\newif\ifarxiv
\arxivtrue
\documentclass[11pt,letterpaper]{article}
\usepackage[margin=1in]{geometry}
\usepackage{xcolor}

\usepackage{times}  
\usepackage{helvet}  
\usepackage{courier}  
\usepackage[hyphens]{url}  
\usepackage{graphicx} 
\usepackage[numbers,sort&compress]{natbib}
\usepackage{caption} 
\usepackage{placeins}
\usepackage{amsmath}
\usepackage{amsthm}
\usepackage{thmtools}
\usepackage{amssymb}
\usepackage{amsfonts}
\newtheorem{lemma}{Lemma}

\usepackage{algorithm}
\usepackage{algorithmic}
\usepackage[hidelinks]{hyperref}
\usepackage{cleveref}
\newcommand{\norm}[1]{\left\lVert#1\right\rVert}
\newcommand{\E}{\mathbb{E}}
\newcommand{\floor}[1]{\left\lfloor#1\right\rfloor}
\newcommand{\ceil}[1]{\left\lceil#1\right\rceil}
\newcommand{\parentheses}[1]{\left(#1\right)}
\newcommand{\angles}[1]{\left\langle#1\right\rangle}

\newcommand{\set}[1]{\left\{#1\right\}}

\newcommand{\Var}[1]{\mathrm{Var}{#1}} 

\newif\ifcomm
\commfalse
\ifcomm
\newcommand{\SV}[1]{\textcolor{blue}{SV:~#1}}
\newcommand{\ran}[1]{\textcolor{red}{Ran:~#1}} 
\newcommand{\YBI}[1]{\textcolor{green}{YBI:~#1}} 
\newcommand{\MM}[1]{\textcolor{purple}{MM:~#1}} 
\newcommand{\AP}[1]{\textcolor{violet}{AP:~#1}} 
\else
\newcommand{\SV}[1]{}
\newcommand{\ran}[1]{}
\newcommand{\YBI}[1]{}
\newcommand{\MM}[1]{}
\newcommand{\AP}[1]{}
\fi
\title{Entropy-Constrained Adaptive Stochastic Quantization}
\date{}
\author{\makebox[\textwidth][c]{\small Ran Ben Basat\textsuperscript{1,2}\quad Yaniv Ben-Itzhak\textsuperscript{2}\quad Michael Mitzenmacher\textsuperscript{3}\quad Shay Vargaftik\textsuperscript{2}}\\[0.5em]
\small \textsuperscript{1}University College London\quad
\textsuperscript{2}VMware Research by Broadcom\quad
\textsuperscript{3}Harvard University}

\newcommand{\vnmse}{\mathrm{vNMSE}}

\newcommand{\appendixlocation}[1]{\ifarxiv Appendix~\ref{#1}\else the supplementary material, Appendix~\ref{#1}\fi}

\begin{document}

\maketitle

\begin{abstract}
Adaptive stochastic quantization (ASQ) is a recently introduced quantization approach that optimizes the Mean Squared Error (MSE) for a given input while preserving unbiasedness. It is designed to alleviate the communication and memory bottlenecks of modern data and machine learning workloads, including model, gradient, and KV-cache compression and nearest-neighbor search. 
Further, practical systems can then compress quantized data with a lossless entropy encoder. 
However, existing unbiased methods, including ASQ, choose their quantization values without considering this later encoding stage, leaving accuracy \mbox{on the table.}

We formulate the Entropy Constrained Adaptive Stochastic Quantization (ECASQ) problem, which jointly selects adaptive quantization values to minimize MSE under an entropy budget and an unbiasedness constraint. We give an optimal dynamic program with $O(sd^2)$ time and $O(d^2)$ space for a length-$d$ vector and at most $s$ quantization values, and a GPU-friendly approximate dynamic program with $O(sd^2)$ time and $O(d)$ space. The approximation guarantees that the solution has an MSE no larger than the optimal solution that uses one fewer bit of entropy per entry. We also provide an iterative refinement procedure for the approximation solution that, in our experiments, yields near-optimal results while retaining a substantial speed advantage over our solver for \mbox{the optimal solution.}
\end{abstract}
\section{Introduction}
\vspace*{1mm}
\label{sec:introduction}

The training and deployment of machine learning (ML) models are often constrained by memory capacity, computational resources, and network bandwidth. To alleviate these constraints, quantization has become a cornerstone technique. By mapping high-precision values to low-width representations, quantization facilitates effective compression strategies across various ML workloads, including gradient compression~\cite{ramezani2021nuqsgd,konevcny2018randomized,vargaftik2022eden,han2024beyond}\ifarxiv\ and communication-efficient distributed training~\cite{li2024thc,chen2024ml,han2026dynamiq,zhao2026mlforml}\fi, post-training quantization~\cite{frantar2023optq,jeon2023frustratingly}, and \mbox{KV-cache footprint reduction} \cite{sheng2023flexgen,liu2024minicache}\ifarxiv, including disaggregated inference~\cite{zhang2025hack}\fi.

Three properties are especially desirable in this setting: (1) unbiasedness, (2) adaptivity to the input, and (3) optimization for the size of the entropy-encoded output. As we describe next, existing methods do not provide \mbox{all three simultaneously.}

Let $X=\angles{x_1,\ldots,x_d}\in\mathbb R^d$ be the input, and let $Q=\set{q_1,\ldots,q_m}\subset\mathbb R$, with $q_1<\cdots<q_m$, be an ordered \emph{codebook} of $m$ quantization values. To enable unbiasedness, a valid $Q$ must span $X$: $q_1\le\min_i x_i$ and $q_m\ge\max_i x_i$. Stochastic Quantization (SQ) rounds each $x_i$ between its nearest values in $Q$, producing $\widehat X=\angles{\widehat x_1,\ldots,\widehat x_d}$ with $\E[\widehat x_i]=x_i$. \mbox{Its distortion is} $MSE(Q,X)=\E\big[\lVert X-\widehat X\rVert^2\big]$.

\textbf{Coordinatewise \textbf{unbiasedness}}, $\E[\widehat X]=X$, is desired for different applications, including for fast convergence under gradient compression~\cite{NIPS2017_6c340f25,castroquartet}. In Distributed Mean Estimation~\cite{pmlr-v70-suresh17a}, it also makes errors cancel, allowing the MSE to decrease inversely with the \mbox{number of participants}~\cite{vargaftik2021drive, basat2024quick}.\footnote{We note that there are other methods for achieving unbiased quantization, such as \ifarxiv shared-randomness protocols~\cite{ben2021send}, \fi subtractive dithering~\cite{rapp2019estimation} and Adaptive Unbiased Quantization\ifarxiv~\cite{basat2025better}\fi. None of these methods are optimized \mbox{for entropy encoding.}}

\textbf{Adaptive methods} choose $Q$ for the particular input $X$ and improve substantially over predetermined codebooks~\cite{zhang2017zipml, faghri2020adaptive, ben2024optimal}. Adaptive Stochastic Quantization (ASQ) minimizes $MSE(Q,X)$ over $|Q|\le s$ given a bound $s$~\cite{zhang2017zipml}; it is therefore optimal for fixed-length value identifiers, or when no entropy encoding follows. ASQ admits an \mbox{optimal $O(d\cdot s)$-time algorithm} \cite{ben2024optimal}.

\textbf{Lossless entropy encoding}, which is often follows quantization by practical systems, assigns shorter bit strings to more frequent values. For $q\in Q$, let $f_q=d^{-1}\sum_{i=1}^d\Pr[\widehat x_i=q]$. We define the average-entry entropy (in bits) as
$H(\widehat X)=-\sum_{q\in Q}f_q\log_2f_q$, with $0\log_2 0=0$. Equivalently, this is the entropy of $\widehat x_I$ for a uniformly random coordinate $I$.

Entropy-aware adaptive methods include Entropy-Constrained Scalar and Vector Quantization~\cite{gish1968asymptotically,chou1989entropy}, broader vector-quantization frameworks~\cite{gersho1992vector}, and neural compression~\cite{balle2017end}. These methods do \mbox{not enforce unbiasedness.}

A separate line of research in distributed optimization includes quantization schemes in which unbiased quantization is followed by entropy encoding. Notable examples include QSGD~\cite{NIPS2017_6c340f25} and Distributed Mean Estimation with limited communication~\cite{pmlr-v70-suresh17a}, but they are inherently \textit{non-adaptive}, i.e., they rely on predetermined quantization grids (e.g., values scaled globally by a vector norm), which are generally less accurate and can result in excessive codebook sizes.

\ifarxiv
Randomized-transform approaches such as DRIVE and EDEN likewise apply scalar
quantizers whose levels are prescribed up to scaling, rather than solving the
input-specific ASQ codebook problem. Recent work further clarifies this line's
relationship to TurboQuant and analyzes when fast randomized Hadamard
transforms can replace uniform random rotations~\cite{vargaftik2021drive,vargaftik2022eden,benbasat2026turboquant,benbasat2026hadamard}.
\fi

This motivates the \emph{Entropy Constrained Adaptive Stochastic Quantization} (ECASQ) problem. In addition to $X$, its inputs are a \textit{permissible} alphabet $P\subset\mathbb R$ of size $p$ (assumed to contain an encompassing codebook), a maximum codebook size $s\le p$, and an average-entry entropy bit-budget $b\ge0$. Specifically, ECASQ solves
\begin{equation*}
\begin{aligned}
&\min_{Q \subseteq P} \qquad\qquad MSE(Q,X) \\
&\text{subject to}  \  \qquad |Q| \leq s, \,\, H(\widehat X) \leq b,\,\, \E[\widehat X]=X.
\end{aligned} \quad.
\end{equation*}
In addition to the entropy constraint, constraining the codebook size $s$ is important to limit its representation overhead when compressing vectors of limited size. Also, working with smaller codebooks results in a faster ECASQ solution as well as quicker quantizing and dequantizing.

\textbf{Solution Overview.}
The first step towards solving ECASQ is to introduce a Lagrange multiplier (as commonly done by previous entropy-encoding-aware works, e.g.,~\cite{gish1968asymptotically, chou1989entropy}), which allows us to consider, for some $\lambda \ge 0$, the relaxed problem, which we name $\lambda$-ECASQ.
\begin{equation*}
\begin{aligned}
&\min_{Q \subseteq P} \qquad\qquad MSE(Q,X)+ \lambda\cdot H(\widehat X) \\
&\text{subject to}  \  \qquad |Q| \leq s, \,\, \E[\widehat X]=X.
\end{aligned}\label{eq:obj} \quad.
\end{equation*}
Intuitively, $\lambda=0$ is the classic ASQ problem, while large $\lambda$ values force the algorithm to prioritize low entropy over reducing error. Accordingly, one 
can vary $\lambda$ to retrieve the convex hull of the Pareto-optimal MSE-Entropy trade-off, under the natural assumption that each entry is stochastically \mbox{quantized to one of the encompassing quantization values.}

We note that, generally, this Lagrangian optimization may not be able to find an optimal solution to the ECASQ problem for a given value $b$. Specifically, a respective $\lambda$ value that results in the exact entropy constraint may not exist (but if some $\lambda$ achieves a value $b$, that solution is optimal for that $b$). We expand on this issue in~\Cref{sec:discussion} and show that if `time-sharing' (running two quantizers obtained from solving $\lambda$-ECASQ with two different $\lambda$ values, on a suitable random partition of the input) is allowed, then we recover the ability to provide the optimal MSE for any entropy constraint $b$. In practice, however, the single Lagrangian solution is generally optimal or very close to optimal and avoids the need for two quantizers.

We first explain why the dynamic program (DP) used by optimal ASQ methods such as ZipML~\cite{zhang2017zipml} and QUIVER~\cite{ben2024optimal} cannot be extended to consider entropy encoding.
Instead, we present a novel DP that allows solving $\lambda$-ECASQ optimally. While a naive implementation of the new DP requires $O(p^3\cdot s + d)$ time and $O(p^2\cdot s + d)$ space, we show how to use existing DP acceleration and space-saving techniques such as SMAWK~\cite{aggarwal1986geometric} and Hirschberg~\cite{hirschberg1975linear} to lower the \mbox{asymptotics to $O(p^2\cdot s + d)$ time and $O(p^2 + d)$ space.}

To further reduce space and enable running on large $p$, we present an approximation algorithm which guarantees that it is at most $1$ bit away from being optimal; that is, with an entropy bound of $b>1$ bits, the approximate solution has an MSE that is not larger than that of an optimal $\lambda$-ecasq solution with an entropy bound of $b-1$ bits and the same $\lambda$ value.
This new approximation algorithm requires $O(p^2\cdot s + d)$ time and $O(p\cdot s + d)$ space, where the latter can be reduced to $O(p + d)$ using Hirschberg's method.
While the asymptotic running time of this algorithm matches that of our optimal DP, it is orders of magnitude faster in practice due to being GPU-friendly and space efficient. 



Finally, we show how the approximation algorithm's accuracy can be further improved by cascading the approximate optimization with a Lloyd-max-style iterative refinement phase. 
Through extensive evaluation over multiple input distributions, we demonstrate that the result is near-optimal after a handful of steps while remaining much faster and space-efficient than the \mbox{optimal DP variant.}

\section{Preliminaries}\label{sec:preliminaries}
Given an input $X\in\mathbb R^d$ and a set of quantization values $Q\subset R$, Stochastic Quantization (SQ) returns a quantized vector $\widehat X = \langle\widehat x\rangle_{x\in X}$ such that each $x\in X$ is unbiasedly quantized between its encompassing values in $Q$. 
In line with previous works (e.g.,~\cite{ben2024optimal,zhang2017zipml,faghri2020adaptive}), we hereafter assume that $X$ is \emph{sorted} (otherwise, we sort a copy of $X$ and keep the original indices to quantize entries in the initial order). We further assume that the entries of $X$ are distinct (otherwise, we can work with its weighted histogram).
That is, denoting by $a_x=\max\set{q\in Q\mid q\le x}$ and $b_x=\min\set{q\in Q\mid q\ge x}$ the relevant entries in $Q$, we have that $(\widehat x=a_x)$ if $(a_x=b_x)$ or $\widehat x = \begin{cases}
    b_x & \mbox{w.p. $(x-a_x)/(b_x-a_x)$} \\ 
    a_x & \mbox{otherwise}
\end{cases}$ if $(a_x\neq b_x)$.
Since $\Var[\widehat x] = \Pr[\widehat x = b_x]\cdot (b_x-x)^2 + \Pr[\widehat x = a_x]\cdot (a_x-x)^2 = (b_x-x)(x-a_x)$, we have that the Mean Squared Error (MSE) of the quantization is 
\ifarxiv
\[
MSE(Q,X) = \sum_{x\in X}(b_x-x)(x-a_x).
\]
\else\mbox{$MSE(Q,X) = \sum_{x\in X}(b_x-x)(x-a_x).$}
\fi
\ifarxiv
An alternative error metric is the vector-Normalized MSE (vNMSE)~\cite{vargaftik2021drive,benbasat2026hadamard,benbasat2026turboquant}, which is defined as
\else
An alternative error metric is the vector-Normalized MSE (vNMSE)~\cite{vargaftik2021drive}, which is defined as
\fi
\ifarxiv
\[
    vNMSE(Q,X) = MSE(Q,X) / \norm{X}^2.
\]
\else
$
    vNMSE(Q,X) = MSE(Q,X) / \norm{X}^2.
$
\fi
vNMSE is often a more convenient metric as it is scale free (i.e., $vNMSE(Q,X)=vNMSE(c\cdot Q,c\cdot X)$ for any $c\neq 0$) which allows comparing methods for \mbox{various input dimensions.}

Adaptive SQ (ASQ) is the problem where we are given just $X\in\mathbb R^d$ and $s\ge 2$ and are asked to output $Q$ of size $s$ that minimizes $MSE(Q,X)$ (or equivalently, $vNMSE(Q,X)$) for \mbox{the specific $X$.}

Optimal ASQ methods such as~\cite{zhang2017zipml,ben2024optimal} operate by solving the following dynamic program (DP); they define $dp_{ASQ}(i,j)$ as the minimal MSE attainable by stochastically quantizing the $j$ smallest entries in $X$ using $i$ quantization values.
They define $C[k,j]=\sum_{\ell= k}^j(x_j-x_\ell)(x_\ell-x_k)$ as the sum of variances of all entries in $[x_k,x_j]$ assuming that they are quantized between $x_k$ and $x_j$.
This formulation can then be solved using the recurrence for $dp_{ASQ}(i,j):$
{
\begin{align}
     \begin{cases}
        \min\limits_{k\in\{1,\ldots,j\}} dp_{ASQ}(i-1,k) + C[k,j] & \mbox{If $i > 2$}\\
        C[k,j]        & \mbox{Otherwise}
    \end{cases}.\label{eq:asqdp}
\end{align}
}
\!Intuitively, this assumes that $x_j$ is a quantization value and finds the optimal placement $x_k$ of the next value.
This dynamic program crucially relies on the fact that there exists an optimal solution $Q^*$ in which $Q^*\subseteq X$, i.e., where all quantization values are entries {in the input} \cite{zhang2017zipml,ben2024optimal}. This problem can be solved optimally in $O(s\cdot d)$ time and space~\cite{ben2024optimal}.
For an integer $n$, we denote $[n]=\{1,\ldots,n\}$.

\section{Entropy Constrained ASQ}
To optimize the set of quantization values $Q$ when an entropy encoding step will follow, we extend the ASQ problem as follows.
The input for the Entropy Constrained Adaptive Stochastic Quantization (ECASQ) problem now includes (in addition to the parameter $s$ that bounds $|Q|$) an entropy bound $b\in\mathbb R^+$ and a set of \textit{permissible} quantization values $P\subset R$. We add the constraints that $H(\widehat X)\le b$ and that $Q\subseteq P$. Here, $H(\widehat X)$ is the entropy of the random vector of the quantized entries. Namely, for each $q\in Q$, let $f_q=\sum_{x\in X}\Pr[\widehat x = q]$ denote the expected frequency of $q$. Then the \textit{average entry} entropy is defined as 
$$
H(\widehat X) = \sum_{q\in Q} -f_q/d \cdot \log_2 f_q/d.
$$
For example, consider $X=\angles{0,1,10}$ and $Q=\set{0,10}$; then $0$ is surely quantized to $0$ and $10$ is quantized to $10$, while the middle entry is quantized to $0$ with probability $9/10$ and to $10$ otherwise. The entropy is therefore $H(\widehat X)=-19/30\log_2(19/30)-11/30\log_2(11/30)\approx 0.948$ \mbox{bits per entry.}

The introduction of $P$ is necessary because, unlike in the ASQ case, the optimal solution may not be a subset of the input, and it also allows us to impose practical constraints on the encoding (e.g., by defining $P$ as all values representable in BF16 or FP32).
Notice that ECASQ generalizes ASQ, which is the special case $b=\log_2 s$ and $P=X$. Although one could remove the separate cardinality parameter by setting $s=|P|$, doing so ignores the size required for representing the codebook. The decoder must know both the selected dictionary $Q$ and the entropy-code metadata; a cap on $s$ therefore gives a direct, implementation-independent bound on this overhead. 
If a stored reconstruction value and its entropy-model descriptor use $w$ and
$c$ bits, respectively, their metadata costs $(w+c)\lvert Q\rvert/d$ bits per
entry.  Since $Q$ is ordered, a level's position implicitly supplies its
symbol-to-value association, so no additional permutation is needed.  Our
experiments use $w=c=16$.

Towards solving ECASQ, we follow previous entropy-aware quantization works (e.g.,~\cite{chou1989entropy}) and introduce a Lagrangian multiplier $\lambda'>0$ to write a revised cost function of 
$MSE(Q,X) + \lambda'\cdot H(\widehat X)$.
In practice, it is often comfortable to use 
$vNMSE(Q,X) + \lambda\cdot H(\widehat X)$ as the cost, which is equivalent to the above when $\lambda = \lambda' / \norm{X}^2$. Notice that this still generalizes ASQ (for $\lambda=0$) and that by a binary search for $\lambda$ we can respect \mbox{the entropy constraint.}

We note that the standard ASQ DP~\eqref{eq:asqdp} is not suitable for the Lagrangian parameterization. This is because the new cost function is not separable, and the expected frequency (and thus, entropy) of a quantization value $x_k$ depends on both adjacent quantization values, not just the one on its right ($x_j$).
Accordingly, we design our optimal and approximate DPs to \mbox{overcome this challenge.}

In what follows, we derive optimal and approximate algorithms for ECASQ. In~\Cref{sec:optDP}, we present an optimal algorithm for the problem that runs in $O(d^2\cdot s)$ time and require $O(d^2)$ space.
Subsequently,~\Cref{sec:apxDP} features an approximation algorithm that requires just $O(d)$ space while producing a solution that, for $b>1$, has an MSE which is upper bounded by the optimal solution \mbox{with $b-1$ bits.}
\ifarxiv
The generalization to $P\neq X$ appears in \appendixlocation{app:general-p}.
\else
In the interest of space, the generalization to $P\neq X$ appears in \appendixlocation{app:general-p}.
\fi

\section{The Optimal Dynamic Program}\label{sec:optDP}
The intuition behind our optimal DP is that, as expected frequency of a quantization value in the quantized vector $\widehat X$ only depends on the two adjacent values in $Q$, we can account for its contribution to the entropy if we track the last \mbox{two quantization values.}

For ease of presentation, we first assume that $P=X$, i.e., that the algorithm may only place quantization values on input entries, same as in \mbox{the ASQ methods.}


We present a new dynamic program that uses two helper arrays, \mbox{for any $1<a<b\in \set{1,\ldots,d}$:}
\begin{align*}
    U[1,b] &= \sum_{i=1}^b \frac{x_i-x_1}{x_b-x_1} \quad ,
    \quad U[a,b] = \sum_{i=a+1}^b \frac{x_i-x_a}{x_b-x_a} \quad \\
    \quad D[1,b] &= \sum_{i=1}^b \frac{x_b-x_i}{x_b-x_1}\quad ,
    \quad D[a,b] = \sum_{i=a+1}^b \frac{x_b-x_i}{x_b-x_a} \ .
\end{align*}
That is, $U[a,b]$ is the sum of all "round-up" probabilities (and $D[a,b]$, the round-down) of entries in the range, assuming two quantization values at $x_a,x_b$ and none in between. The intuition is that the frequency of a value at $x_\ell$, assuming that $x_k,x_\ell,x_j\in Q$ are consecutive, is $U(k,\ell) + D(\ell,j)$ in expectation. \mbox{For all $n$, we set $U[n,n]=D[n,n]=0$.}

For the new dynamic program, we define $dp_{OPT}(i,\ell,j)$ to be the minimal cost for quantizing the $\ell$ smallest entries using $i-1$ quantization values that include $x_\ell$, \emph{plus} the MSE for quantizing $\set{x_\ell,\ldots,x_j}$ assuming that $x_\ell,x_j\in Q$ with no quantization values in between. 
This is \mbox{illustrated in~\Cref{fig:optdp}.}

\begin{figure}
    \centering
    \ifarxiv
    \includegraphics[width=0.58\linewidth]{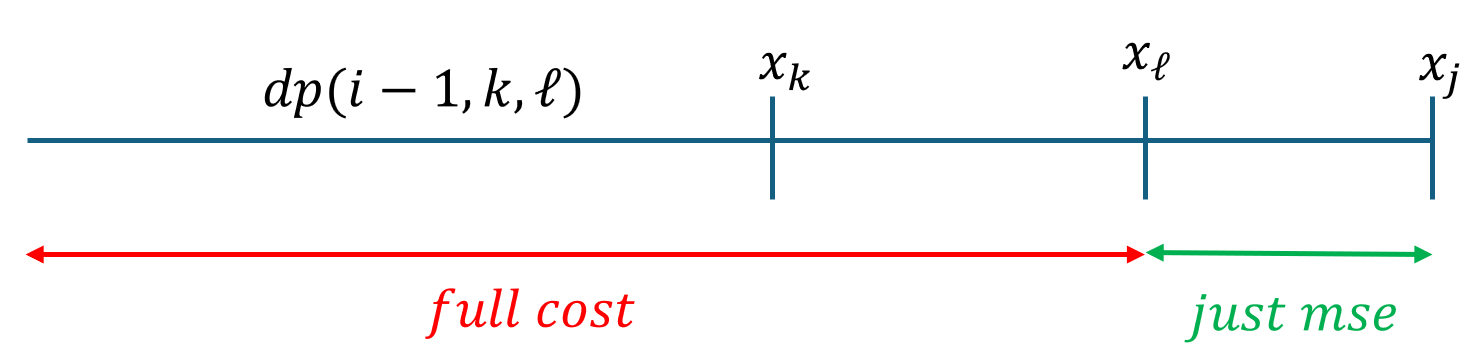}
    \else
    \includegraphics[width=.8\linewidth]{Screenshot_2026-02-20_093525__1_.png}
    \fi
    \caption{The structure of $dp_{OPT}(i,\ell,j)$: we look for the $k$ that minimizes the cost of the prefix up to $x_\ell$ and only the MSE for $[x_\ell,x_j]$.}
    \label{fig:optdp}
\end{figure}

For $y\in[0,1]$, let $h(y)=\begin{cases}
    -y\log_2 y & \mbox{if \ $y >0$}\\
    0 & \mbox{otherwise}  \ 
\end{cases}.$ \\ Assuming we solve $dp_{OPT}(i,\ell,j)$ for all $i\in[s], \ell,j\in[d]$, we have that the optimal cost is:

$$
\min_{\ell\in[d]} \Bigg(\ dp_{OPT}(i,\ell,d) + \lambda\times h\parentheses{\frac{U(\ell,d)}{d}}\ \Bigg) .
$$

$dp(i,\ell,j)$ is solved using the recurrence for $i\ge 3$:

$$
dp_{OPT}(i,\ell,j) = 
    C[\ell,j] +  \min\limits_{k\in[\ell]} \ T_{i,\ell,j}(k)\ ,
$$
\noindent where
{\small
\[
T_{i,\ell,j}(k)\triangleq dp_{OPT}(i-1,k,\ell)
+\lambda\times h\parentheses{\frac{U(k,\ell)+D(\ell,j)}{d}}.
\]
}
The stopping condition for $i=3$ is
\ifarxiv
\[
dp_{OPT}(3,\ell,j)=C[1,\ell]+C[\ell,j]
+\lambda\!\left[
h\parentheses{\frac{D(1,\ell)}{d}}
+h\parentheses{\frac{U(1,\ell)+D(\ell,j)}{d}}
\right].
\]
\else
\[
\resizebox{\columnwidth}{!}{$
\begin{aligned}
dp_{OPT}(3,\ell,j)={}&C[1,\ell]+C[\ell,j]\\
&+\lambda\times\left[
h\parentheses{\frac{D(1,\ell)}{d}}
+h\parentheses{\frac{U(1,\ell)+D(\ell,j)}{d}}
\right].
\end{aligned}
$}
\]
\fi

Naively, computing the value of each $(i,\ell,j)$ entry requires $O(d)$ time (for checking all options for $k$), while there are $s\cdot d^2$ cells in the DP, giving runtime and space bounds of $O(d^3\cdot s)$ and $O(d^2\cdot s)$, respectively. Below, we improve these to $O(d^2\cdot s)$ and $O(d^2)$ using general DP machinery.
We note that the recursion considers duplicates in $Q$ as $\ell\in[\ell]$; one can remove these to obtain a duplicate-free $Q$ with $|Q|\le s$.

\subsection{Optimizing runtime using SMAWK}\label{sec:smawk}
The SMAWK algorithm~\cite{aggarwal1986geometric} is a general technique for accelerating DPs.
The input for the algorithm is a matrix $M\in \mathbb R^{d\times d}$ that satisfies the quadrangle inequality (also known as Monge property):
$$
\forall \texttt a{\le}\texttt  b{\le}\texttt  c{\le}\texttt d: M_{\texttt a,\texttt  c}+M_{\texttt  b,\texttt d} \le M_{\texttt a,\texttt d} + M_{\texttt  b,\texttt  c}.
$$
The algorithm then finds the row minima in $O(d)$ time. That is, it outputs $k_1,\ldots,k_d$ such that $k_y$ is the smallest value in the \mbox{$y$'th row.}

An important property of the quadrangle inequality is stated in the following lemma (proved in \appendixlocation{app:quadrangle}).

\begin{restatable}{lemma}{quadrangleg}\label{lem:quadrangle_g}
    Let $g$ be a concave function and let $v, m$ be monotonically decreasing functions then $G(k,j)\triangleq g(v(k)+m(j))$ satisfies \mbox{the quadrangle inequality.}
\end{restatable}

%
%
%

{
    Fix $i$ and $\ell$. Let the admissible predecessor indices be
$K_\ell=\{k:1\le k<\ell\}$ and the admissible right endpoints be
$J_\ell=\{j:\ell<j\le d\}$. Denote
\[
A(k)=dp_{\mathrm{OPT}}(i-1,k,\ell)
\]
and
\[
G(k,j)=\lambda\cdot h\!\left(\frac{U(k,\ell)+D(\ell,j)}{d}\right).
\]
For our DP, define
\[
M_{k,j}\triangleq T_{i,\ell,j}(k)=A(k)+G(k,j).
\]

The matrix $M$ is not necessarily Monge in the natural order of the
columns, since $D(\ell,j)$ is monotone increasing in $j$. We thus
apply SMAWK on reversed column order. Let
\[
\rho_\ell(r)=d+\ell+1-r,\qquad r\in J_\ell,
\]
be the order-reversing bijection of $J_\ell$, and define
$
\widetilde M_{k,r}=M_{k,\rho_\ell(r)} .
$
The following is proved in \appendixlocation{app:quadrangleM}.

\begin{restatable}{lemma}{quadrangleM}\label{lem:quadrangleM}
For fixed $i$ and $\ell$, the column-reversed matrix
$\widetilde M$ satisfies \mbox{the quadrangle inequality.}
\end{restatable}

Since $\widetilde M$ is Monge, so is its transpose. Therefore SMAWK
can be run on $\widetilde M^{\mathsf T}$ to compute, in
$O(|K_\ell|+|J_\ell|)=O(d)$ time, a minimizing predecessor
\[
k^\ell_j \in \arg\min_{k<\ell} M_{k,j}
\]
for every $j\in J_\ell$. We then fill the DP entries as
\[
dp_{\mathrm{OPT}}(i,\ell,j)
=
C[\ell,j]+T_{i,\ell,j}(k^\ell_j).
\]
For $i>3$, we invoke SMAWK for every $\ell\in[d]$ and use its output $\set{k_j^\ell}$ to fill in $d$ entries in the matrix, as:
{
$$
dp_{OPT}(i,\ell,j) = C[\ell,j] + T(k_j^\ell).
$$
}

Overall, as we invoke the $O(d)$-time SMAWK algorithm for each $i\in[s],\ell\in[d]$, the \mbox{runtime becomes} $O(d^2\cdot s)$.

\subsection{Optimizing space using Hirschberg}\label{sec:hirschberg}
We now explain how to apply Hirschberg's algorithm~\citeyearpar{hirschberg1975linear} to \mbox{our optimal DP.}

The first step is to observe that while we defined $dp_{OPT}(i,\ell,j)$ to be the minimal cost for quantizing the $\ell$ smallest entries in $X$ (henceforth referred to as the ``forward algorithm''), an equivalent algorithm can define $dp_{OPT}^B(i,\ell,j)$ to be the cost for the largest entries (the "backward algorithm").
%
%
Specifically, we define $dp_{OPT}^B(i,\ell,j)$ as the minimal cost for quantizing the $d-j$ \emph{largest} entries using $i-1$ quantization values that include $x_j$, \emph{plus} the MSE for quantizing $\set{x_\ell,\ldots,x_j}$ assuming that $x_\ell,x_j\in Q$ with no quantization values in between. 
The recurrence is then, for $i\ge 3$:

$$
dp_{OPT}^B(i,\ell,j) = 
    C[\ell,j] +  \min\limits_{k\in[d]\setminus [j-1]} \ T^B_{i,\ell,j}(k)\ ,
$$
\noindent where
{\small
\[
T^B_{i,\ell,j}(k)\triangleq dp_{OPT}^B(i-1,j,k)
+\lambda\times h\parentheses{\frac{U(\ell,j)+D(j,k)}{d}}.
\]
}

The Hirschberg method follows a divide-and-conquer approach: 
it uses the forward algorithm to compute  $dp_{OPT}(i,\ell,j)$, for all $\ell,j\in[d]$ and $i\in[\floor{s/2}]$.
It then uses the backward algorithm to compute $dp_{OPT}^B(i,\ell,j)$ for all $i,\ell,j\in[d],\in[\ceil{s/2}]$.
Importantly, the algorithm only keeps a pair of layers $(i,i+1)$ as it progresses, keeping the memory \mbox{bounded by $O(d^2)$.}

Next, denoting $t=\floor{s/2}$, it finds the values $\ell^*,j^*$ that minimize
$$
dp_{OPT}(t+1,\ell^*,j^*) + dp_{OPT}^B(s-t+1,\ell^*,j^*) - C[\ell^*,j^*] \  \ .
$$
Intuitively, the optimal minimum is obtained when $\ell^*,j^*$ are the middle quantization values in \mbox{the optimal solution.}

$dp_{OPT}(\floor{s/2},\ell^*,j^*)$ and $dp_{OPT}^B(\ceil{s/2},\ell^*,j^*)$ are then recursively computed while reusing memory across the nested calls (i.e., we keep just $\ell^*,j^*$ from the initial run and once $dp_{OPT}(\floor{s/2},\ell^*,j^*)$ terminates, we reuse the memory except the inner minimizer indices.)

We state the resulting asymptotics (proof in \appendixlocation{app:Hirschberg_opt}).

\begin{restatable}{lemma}{Hirschbergopt}\label{lem:Hirschberg_opt}
Using Hirschberg's algorithm, the space complexity of our DP is $O(d^2)$ while the time \mbox{complexity remains $O(d^2\cdot s)$.}
\end{restatable}

\paragraph{General permissible sets.}
The optimal DP extends from $P=X$ to any sorted permissible set
$P=\langle p_1,\ldots,p_p\rangle$ by indexing states by permissible values and
adjusting the boundary conditions (see \appendixlocation{app:general-p} for the full construction). The resulting optimal algorithm runs in
$O(p^2\cdot s+d)$ time and \mbox{uses $O(p^2+d)$ space.}

\section{\mbox{A Space-Efficient Approximation Algorithm}}\label{sec:apxDP}
The optimal DP from~\Cref{sec:optDP} keeps enough state to charge the entropy contribution of each quantization value only after both of its neighboring quantization values are known. We now present a more space-efficient alternative. The main idea is to optimize a slightly stronger entropy surrogate whose contribution is local to each interval between consecutive quantization values. This removes the need to keep two adjacent quantization values \mbox{in the state.}

Let $Q=\{q_1<\ldots<q_t\}$ be a feasible set of quantization values, with $q_1\le x_1$, $q_t\ge x_d$, and $t\le s$. For each coordinate $x$, define its interval index $I_x\in\{1,\ldots,t-1\}$ such that
$q_{I_x}<x\le q_{I_x+1}$.\footnote{If $x=q_1$, \mbox{we define $I_{x}=1$.}}
Let $B_x\in\{0,1\}$ indicate the stochastic rounding decision, where $B_x=1$ if $x$ is rounded up to $q_{I_x+1}$ and $B_x=0$ otherwise. Thus $\widehat x=q_{I_x+B_x}$.

The approximation comes from optimizing the joint entropy $H(I,B)$ rather than $H(\widehat X)$, where the randomness is over a uniformly chosen coordinate and the stochastic rounding. For a fixed $Q$, the map $(I,B)\mapsto (\widehat X,B)$ is a bijection and therefore
\[
H(I,B)=H(\widehat X,B)=H(\widehat X)+H(B\mid \widehat X).
\]
Since $B=\{B_x\}$ has one bit per entry,
\begin{equation}
    H(\widehat X)\le H(I,B)\le H(\widehat X)+1 .
    \label{eq:apx_entropy_sandwich}
\end{equation}

Before delving into the algorithm, there are several important things to note here: First, finding the minimal MSE solution under the constraint $H(I,B)\le b$ immediately yields an MSE better than the optimal solution with one fewer bit per entry (i.e., for $H(\widehat X)\le b-1$).
This is because, by~\Cref{eq:apx_entropy_sandwich}, the optimal solution for $H(\widehat X)\le b-1$ is a feasible solution for $H(I,B)\le b$; if our algorithm yields a different solution, its MSE \mbox{cannot be larger.}

Second, while we are optimizing $H(I,B)$, this is done to find a good set $Q$, but the quantizer still tries to save a bit (pun intended) by compressing the quantized values in $\widehat X$; thus, our actual \mbox{overhead is $H(B\mid \widehat X)$.}

Third, while this overhead is bounded by one bit per entry, its actual value is dependent on $X$ and $Q$. On one extreme, consider the case where $X$ has $0,1$ and $d/2-2$ occurrences of $(1/2-\epsilon)$ and $(1/2+\epsilon)$ each. If $Q=\set{0,1/2,1}$, then all the $(1/2\pm\epsilon)$ are rounded to $1/2$, giving $H(\widehat X)\approx 0$. In contrast, $H(B\mid \widehat X)\approx 1$, since knowing that $\widehat x=1/2$ reveals little about whether $x=1/2-\epsilon$ or $x=1/2+\epsilon$. For the other extreme, let $X$ have $0,1$ and $d/2-2$ occurrences of $(\epsilon)$ and $(1-\epsilon)$ each. If $Q=\set{0,1/2,1}$, then all the $(\epsilon)$s are rounded to $0$ and the $(1-\epsilon)$ entries are rounded to $1$, so $H(B\mid \widehat X)\approx 0$ since $\widehat x$ implies whether $B_x=0$ or $B_x=1$.
In the evaluation, we quantify the magnitude of $H(B\mid \widehat X)$ in practice for input sizes in which running the optimal \mbox{algorithm is feasible.}

Let $P=\langle p_1,\ldots,p_p\rangle$ be sorted and denote, for $a<b$, $X_{a,b}=\{x\in X\mid p_a<x\le p_b\}$. 
For $b>1$, we define $X_{1,b}=\{x\in X \mid p_1 \le x \le p_b\}$, and for
$1<a<b$, define $X_{a,b}=\{x\in X \mid p_a < x \le p_b\}$. We use this
half-open convention only to avoid double-counting entries that are themselves
permissible quantization values. The MSE contribution is unaffected by the
choice of endpoint convention, since entries equal to a quantization value have
zero variance. We define
\[
C_P[a,b] =
\sum_{x\in X_{a,b}} (p_b-x)(x-p_a),
\]
and similarly
{\small
\[
U_P[a,b]=\sum_{x\in X_{a,b}} \frac{x-p_a}{p_b-p_a},
\qquad
D_P[a,b]=\sum_{x\in X_{a,b}} \frac{p_b-x}{p_b-p_a}.
\]
}
Thus, if $p_k,p_\ell,p_j$ are three consecutive quantization values, the
expected frequency of $p_\ell$ is
$
U_P[k,\ell] + D_P[\ell,j].
$

For a non-first interval $(p_a,p_b]$, the two interval-bit symbols have probabilities $U_P[a,b]/d$ and $D_P[a,b]/d$, so define
\[
    \Phi_P[a,b]
    =h\!\left(\frac{U_P[a,b]}{d}\right)
    +h\!\left(\frac{D_P[a,b]}{d}\right).
\]
For the first interval, entries equal to the first quantization value are assigned to the down symbol, and we use
\[
    \Phi^{\mathrm{first}}_P[a,b]
    =h\!\left(\frac{U_P[a,b]}{d}\right)
    +h\!\left(\frac{D_P[a,b]+n_a}{d}\right),
\]
where $n_a=|\{x\in X\mid x=p_a\}|$.
Consequently, if $Q=\{p_{a_1}<\ldots<p_{a_t}\}$, then
\begin{equation}
    H(I,B)
    =\Phi^{\mathrm{first}}_P[a_1,a_2]
    +\sum_{r=2}^{t-1}\Phi_P[a_r,a_{r+1}] .    \label{eq:apx_additive_entropy}
\end{equation}
This is the separability property that the true entropy $H(\widehat X)$ \mbox{does not have.}

We now write the Lagrangian DP for a fixed multiplier $\lambda\ge0$. Let $L=\{a\in[p]\mid p_a\le x_1\}$ and $R=\{a\in[p]\mid p_a\ge x_d\}$. Define $dp_{APX}(i,j)$ to be the minimum surrogate cost of all partial solutions that use $i$ quantization values and whose rightmost quantization value is $p_j$. The initialization for two quantization values is
\begin{equation}
    dp_{APX}(2,j)
    =
    \min_{\substack{a\in L\\ a<j}}
    \left\{
        C_P[a,j]
        +\lambda\,\Phi^{\mathrm{first}}_P[a,j]
    \right\} .
    \label{eq:apx_base}
\end{equation}
For $i\ge3$, the recurrence is
\begin{equation}
\footnotesize
dp_{APX}(i,j)
=\min_{k<j}\Bigl\{dp_{APX}(i-1,k)
+C_P[k,j]+\lambda\,\Phi_P[k,j]\Bigr\}.
    \label{eq:apx_rec}
\end{equation}
The best value using at most $s$ quantization values is then
\begin{equation}
    \min_{2\le i\le s}\min_{j\in R} dp_{APX}(i,j) .
    \label{eq:apx_final}
\end{equation}
If exactly $s$ values are required, the outer minimum is \mbox{replaced by $i=s$.}

The recurrence has the same structure as the standard ASQ recurrence: once the previous quantization value $p_k$ and the new value $p_j$ are fixed, the cost of the new interval is completely determined. 
In~Appendix~\ref{app:general-p}, using preprocessing of $O(d+p)$ time and space, we show how each of $C_P[k,j]$, $U_P[k,j]$, and $D_P[k,j]$ is computed in $O(1)$ time. Thus, one DP layer costs $O(p^2)$ time, and all $s$ layers cost $O(sp^2+d)$ time. We store only the previous and current DP layers, as well as the $O(p+d)$ prefix information, giving $O(p+d)$ space for the value computation. In particular, when $P=X$, this is $O(sd^2)$ time \mbox{and $O(d)$ space.}

To recover the quantization values while preserving the same asymptotic space bound, we use the Hirschberg-based traceback idea as in~\Cref{sec:optDP}, specialized to the one-dimensional recurrence in~\eqref{eq:apx_rec}. A forward pass computes the best cost up to the middle number of quantization values, a backward pass computes the corresponding suffix costs, and their sum identifies a middle quantization value. Recursing on the two sides reconstructs $Q$ while reusing the same two DP layers. This keeps the peak memory at $O(p+d)$ and the running time at $O(sp^2+d)$. Note that while the runtime is asymptotically equal to that of the optimal algorithm of~\Cref{sec:optDP}, the quadratic space reduction is critical if \mbox{$p$ is large.}

For a target entropy budget $b$, we tune $\lambda$ and keep the lowest-MSE solution whose surrogate entropy $H(I,B)$ is at most $b$. The next lemma summarizes \mbox{the resulting guarantee.}

\begin{lemma}
Let $OPT(\beta)$ be the minimum MSE among feasible quantizers with $H(\widehat X)\le \beta$, and, for $b > 1$, let $Q_{APX}(b)$ be the minimum-MSE quantizer with $H(I,B)\le b$. Then $H(\widehat X_{Q_{APX}(b)})\le b$ and
\[
    MSE(Q_{APX}(b),X)\le OPT(b-1),
\]
where $OPT(\beta)=\infty$ if the budget \mbox{$\beta$ is infeasible.}
\end{lemma}

\begin{proof}
The feasibility of the returned quantizer for the original entropy budget follows directly \mbox{from~\eqref{eq:apx_entropy_sandwich}: $H(\widehat X_{Q_{APX}(b)})\le H(I,B)\le b$.}

Let $Q^-$ be an optimal solution for budget $b-1$. Then $H(\widehat X_{Q^-})\le b-1$. By~\eqref{eq:apx_entropy_sandwich}, the same quantizer satisfies $H(I,B)\le H(\widehat X_{Q^-})+1\le b$, so it is feasible for the surrogate problem with budget $b$. Since $Q_{APX}(b)$ minimizes MSE over the surrogate-feasible quantizers, its MSE is \mbox{at most $MSE(Q^-,X)=OPT(b-1)$.}
\end{proof}
\begin{figure*}[t]
    \centering
    \ifarxiv
    \includegraphics[width=0.82\linewidth]{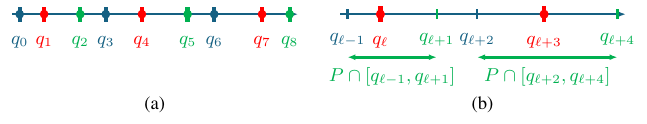}
    \else
    \includegraphics[width=0.592\linewidth]{refinement-figure.pdf}
    \vspace{-4mm}
    \fi
    \caption{Three-color refinement of a quantizer. (a) The quantization values in $Q$ are partitioned by index into $Q_0$ (blue), $Q_1$ (red), and $Q_2$ (green). (b) During the $Q_1$ step, each red value is optimized over its range while neighboring values remain fixed. Consecutive red values have disjoint ranges, so all values in $Q_1$ are updated in parallel; the other two \mbox{substeps are analogous.}}
    \label{fig:refinement}
\end{figure*}

\section{The Refinement Algorithm}\label{sec:refinement}
Due to its space efficiency and GPU-friendliness, the approximation algorithm described above allows running ECASQ on large inputs and data types.
To alleviate the accuracy gap between its output and the optimal algorithm of~\Cref{sec:optDP} while maintaining the speed advantage, we iteratively improve the cost \mbox{while maintaining unbiasedness.}

Our starting point is a feasible solution $Q$ (e.g., one that is generated by the approximation algorithm). \footnote{Without loss of generality we assume that $|Q|=s$; otherwise one can duplicate one of the entries in $Q$ without \mbox{quantizing to it.}}
At each iteration, $Q$ is updated by improving one of three sets: $Q_{0}=\{q_{3\cdot k}\mid k\in\mathbb N, 3k<s\}$, $Q_{1}=\{q_{3\cdot k + 1}\mid k\in\mathbb N, 3k + 1<s\}$, $Q_{2}=\{q_{3\cdot k + 2}\mid k\in\mathbb N, 3k + 2<s\}$. 
When changing the position of $q_\ell$, we use binary search to improve its position in $P\cap [q_{\ell-1},q_{\ell+1}]$, assuming that $q_{\ell-1}$ and $q_{\ell+1}$ remain fixed.
This is achieved by minimizing the sum of variances of all entries $x\in X\cap [q_{\ell-1},q_{\ell+1}]$ plus $\lambda$ times the expected normalized frequency of $q_\ell$, which can be efficiently computed \mbox{using $U_P,D_P$ (\appendixlocation{app:general-p}).}

We choose these three sets as the quantization values in $Q_i$ are independent from each other: moving one does not change the expected frequency of another, and thus we can optimize the values of a $Q_i$ in parallel (while keeping other values fixed) on the GPU, as \mbox{illustrated in~\Cref{fig:refinement}.}

Namely, let $q_{3k+i}$ and $q_{3(k+1)+i}$ be consecutive values in $Q_i$. Then small entries ($X\cap[x_1,q_{3k+i+1}]$) cannot be quantized to $q_{3(k+1)+i}$, while large entries ($X\cap[q_{3k+i+2},x_d]$) cannot be quantized to $q_{3k+i}$, yielding a complete separation of the MSE and \mbox{entropy across $Q_i$.}

While the refinement algorithm is not guaranteed to reach the optimal solution, its output is always at least as good as the starting $Q$, and as we show in the evaluation (\Cref{sec:eval}), it usually is near-optimal while being very quick compared to the \mbox{approximate solution itself.}


\ifarxiv
\else
\fi
\section{Optimality with Two-Way Time-Sharing}\label{sec:discussion}
In this section, we explain the limitations of using the Lagrange relaxation, but also explain how it yields an optimal solution \textit{for any} $b$ when \mbox{time-sharing is allowed.}

\textbf{From Lagrange Multipliers to Entropy Constraints.}
By Everett's theorem~\cite{everett1963generalized}, if $Q_\lambda$ globally
minimizes $\vnmse(Q,X)+\lambda H(\widehat X)$ and happens to satisfy
$H(\widehat X)=b$, then $Q_\lambda$ is globally optimal for the hard-constrained
problem with entropy budget $b$. Thus, whenever the multiplier search hits the
target entropy exactly, no time-sharing is needed. More generally, optimizing
$\vnmse(Q,X)+\lambda H(\widehat X)$ over $\lambda$ recovers the supported points
of the entropy-vNMSE frontier, but it need not return the optimal single
quantizer for every hard entropy budget $b$ (we provide a counterexample in
\appendixlocation{app:timesharing-lagrange}). Time-sharing
fills the entropy levels skipped by the $\lambda$ search: it convexifies the feasible
rate-distortion region, and every boundary point is attained by a mixture of at
most two supported quantizers. Hence a sweep over $\lambda$ recovers an optimal
two-way time-share for \mbox{every entropy constraint.}

\textbf{Quantizing to non-adjacent values in $Q$.}
Our algorithms use ordinary stochastic quantization between the two adjacent
values of $Q$ that encompass each input. For a single quantizer, allowing arbitrary unbiased rounding to non-adjacent values can yield strictly lower vNMSE under the same entropy constraint (we give an example in \appendixlocation{app:timesharing-nonadjacent}).

\textbf{Quantizing to non-adjacent values under two-way time-sharing.}
To regain optimality for non-supported values of $b$ in the strict entropy bound formulation while still solving the $\lambda$-ECASQ problem and rounding to adjacent values in $Q$, we leverage \textit{time-sharing}~\cite{chou1989entropy}. In our application, time-sharing means running two quantizers on a random partition of the input (determined by a shared pseudorandom number generator seed). 
A standard observation is that time-sharing recovers the optimal entropy-MSE tradeoff. Formally, for an entropy bound $b$, let $Q^-$ be the best quantizer with entropy $b^-\le b$ that can be derived from $\lambda$-ECASQ for some $\lambda$, and similarly let  $Q^+$ be the quantizer attainable for some $\lambda'$ with the lowest rate $b^+\ge b$.
Then if for each entry we independently use $Q^+$ with probability $\frac{b - b^-}{b^+-b^-}$, the MSE is guaranteed to be no larger than any single quantizer $Q^*$ with entropy $b$, because the vNMSE is linear in the above probability.
We prove a more general property that is specific to our problem: under time-sharing, it is enough to consider pairs of quantizers that only quantize to adjacent values in their quantization value sets (proof is in \appendixlocation{app:timesharing-proof}).


\begin{restatable}{proposition}{adjacenttimesharingdominates}
\label{prop:adjacent-time-sharing-dominates}
Let $Q_1,Q_2\subseteq P$ satisfy $|Q_1|,|Q_2|\le s$, and let $K_1,K_2$ be
arbitrary unbiased rounding rules into $Q_1,Q_2$. For every
$\alpha\in[0,1]$, there exist two sets $S_1,S_2\subseteq P$, each of size at
most $s$, and $\beta\in[0,1]$, such that ordinary {adjacent stochastic
quantization with $S_1,S_2$ satisfies}
\[
(1) \  \beta R(S_1)+(1-\beta)R(S_2)
\le \alpha R(K_1)+(1-\alpha)R(K_2),
\]
\[
(2) \  \beta D(S_1)+(1-\beta)D(S_2)
\le \alpha D(K_1)+(1-\alpha)D(K_2).
\]
Here, $R(S)$ and $D(S)$ denote the entropy and vNMSE of ordinary adjacent
stochastic quantization with $S$.
\end{restatable}

\noindent We provide an evaluation of the gain attainable by time-sharing in \appendixlocation{app:eval-entropy-frontier}.

\section{Evaluation}\label{sec:eval}

\textbf{Setup.}
Experiments ran on a 12-core Apple M2 with 32 GB RAM under
macOS 15.7.2 (arm64); CPU benchmarks used eight threads with Python 3.9.6 and PyTorch 2.8.0.
 
We evaluate BF16 synthetic vectors of size $d=16{,}384$ over five seeds and
four distributions, together with 100 real-model tensor samples from five
  model families. We compare Optimal ECASQ, Approx ECASQ, five-round refined
Approx ECASQ, and entropy-coded QUIVER, QSGD, and uniform stochastic
quantization.  We report the ideal total rate
$R_{\mathrm{tot}}=H(\widehat X)+32\lvert Q\rvert/d$ bits/value, which includes
a 16-bit BF16 reconstruction value and a 16-bit entropy-model descriptor per
active level.  
Additional details and results appear in \appendixlocation{app:evaluation}.

\begin{figure}[!t]
\vspace*{-0mm}
    \centering
    \includegraphics[width=.789565\columnwidth]{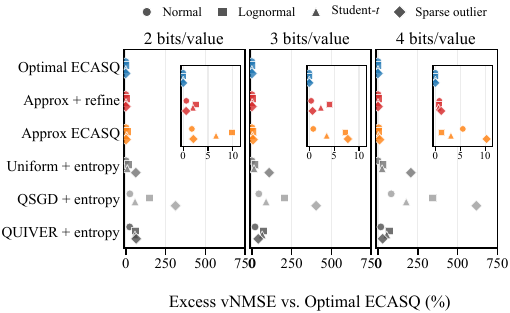}
    \caption{Matched-rate excess vNMSE (five seeds) relative to Optimal ECASQ on synthetic
    BF16 vectors ($d=16{,}384$, at most $s=64$ quantization values).  Lower is better.
    }
    \label{fig:eval-vnmse}
\end{figure}

\textbf{Matched-rate vNMSE.}
Across the 12 distribution-rate comparisons in \Cref{fig:eval-vnmse},
refinement reduces Approx ECASQ's mean excess vNMSE from $4.97\%$ to
$1.37\%$ on the five-seed aggregate curves, with a $4.06\%$ maximum.
Unrefined Approx also outperforms every non-ECASQ baseline in mean at every
rate: its per-rate means are $4.82$--$5.05\%$, versus
$23.23$--$43.36\%$ for the strongest baseline.  Pointwise, it beats QSGD and
QUIVER in all 12 comparisons and uniform quantization in \mbox{11 of 12.}

\begin{figure}[!t]
    \centering
    \vspace{-2mm}
        \includegraphics[width=.546\columnwidth]{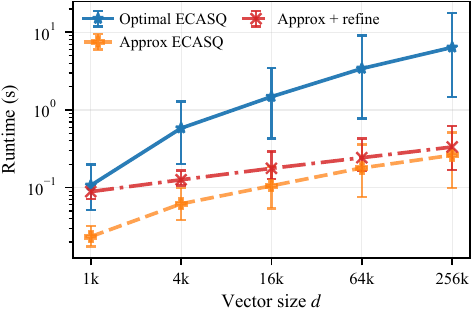}
    \caption{Runtime for BF16 inputs,
    $s=64$, and $\lambda=0.1$. 
    Marks average the four distribution-specific means (five seeds each), and
    error bars span their min-max range. }
    \ifarxiv
    \vspace*{-4mm}
    \fi
    \label{fig:eval-runtime}
\end{figure}

\textbf{Runtime.}
At $d=262{,}144$, \Cref{fig:eval-runtime} shows that Optimal ECASQ, Approx, and refined Approx take $6.39$, $0.26$, and $0.33$ seconds on average: $24\times$ and $19\times$ speedups.  The gains hold for every tested distribution, ranging
from $13$-$34\times$ for Approx and $8$-$28\times$ after refinement.
Other settings are evaluated in \appendixlocation{app:evaluation}.

\textbf{Real models.} We also quantize weight, activation, and KV-cache tensors from Qwen2.5-0.5B/1.5B, Gemma-3-1B-IT, Llama-3.2-1B, and Phi-3.5-Mini, improving througout all tasks. Additional details and results appear in \appendixlocation{app:evaluation}.

\clearpage

\noindent\textbf{AI usage disclaimer.} We have used ChatGPT and Codex to polish the writing of this paper and help with proof verification and simplification as well as evaluation design and execution.

\bibliographystyle{plainnat}
\nocite{ivkin2019know,shahout2023together}
\bibliography{aaai2026}
\setcounter{secnumdepth}{2}
\renewcommand\thesubsection{\thesection.\arabic{subsection}}

\appendix
\section{Proof of~\Cref{lem:quadrangle_g}}\label{app:quadrangle}
We restate the lemma for convenience. \ran{Change to non-increasing}
\quadrangleg*
    \begin{proof}
We show that for every $k<k'$ and $j<j'$,
\[
G(k,j)+G(k',j') \le G(k,j')+G(k',j).
\]

Since $v$ and $m$ are monotonically decreasing,
\[
v(k)\ge v(k')
\qquad\text{and}\qquad
m(j)\ge m(j').
\]
Define
\[
\Delta_v \triangleq v(k)-v(k') \ge 0,
\qquad
\Delta_m \triangleq m(j)-m(j') \ge 0.
\]
Also define
\begin{align*}
    A &\triangleq v(k)+m(j),\qquad\ 
B \triangleq v(k)+m(j'),\qquad\\
C &\triangleq v(k')+m(j),\qquad
D \triangleq v(k')+m(j').
\end{align*}

Then
\[
A=D+\Delta_v+\Delta_m,\qquad
B=D+\Delta_v,\qquad
C=D+\Delta_m,
\]
and therefore
\[
A+D=B+C.
\]

If $\Delta_v+\Delta_m=0$, then $A=B=C=D$, and the claim is immediate.  
Assume now that $\Delta_v+\Delta_m>0$, and let
\[
\theta \triangleq \frac{\Delta_m}{\Delta_v+\Delta_m}\in[0,1].
\]
A direct calculation shows that
\[
B=\theta D+(1-\theta)A,
\qquad
C=(1-\theta)D+\theta A.
\]
Since $g$ is concave, we have
\[
g(B)\ge \theta g(D)+(1-\theta)g(A),
\]
and
\[
g(C)\ge (1-\theta)g(D)+\theta g(A).
\]
Adding these two inequalities yields
\[
g(B)+g(C)\ge g(A)+g(D).
\]
Substituting the definitions of $A,B,C,D$, we obtain
\begin{multline*}
g\bigl(v(k)+m(j')\bigr)+g\bigl(v(k')+m(j)\bigr)\\
\ge
g\bigl(v(k)+m(j)\bigr)+g\bigl(v(k')+m(j')\bigr).
\end{multline*}
By the definition of $G$, this is exactly
\[
G(k,j')+G(k',j)\ge G(k,j)+G(k',j'),
\]
or equivalently,
\[
G(k,j)+G(k',j') \le G(k,j')+G(k',j).
\]
Hence $G$ satisfies \mbox{the quadrangle inequality.}
\end{proof}

\section{Proof of~\Cref{lem:quadrangleM}}\label{app:quadrangleM}
We restate the lemma for convenience. 
\quadrangleM*
\begin{proof}
It suffices to prove the claim for the entropy term, since adding the
row-dependent term $A(k)$ preserves \mbox{the quadrangle inequality.}

Define
\[
v(k)\triangleq U(k,\ell),
        \qquad
\widetilde m(r)\triangleq D(\ell,\rho_\ell(r)),
\]
and
\[
g(z)\triangleq \lambda h(z/d).
\]
Because $h(0)=0$ and $h(y)=-y\log_2 y$ for $y>0$, the function $h$ is concave on $[0,1]$. Moreover, every argument used below satisfies \[0\le U(k,\ell)+D(\ell,\rho_\ell(r))\le d.\]
Therefore, $g(z)=\lambda h(z/d)$ is concave on the entire range of arguments appearing in the matrix.

We first show that $v(k)$ is monotone decreasing in $k$. For
$k<k'<\ell$,
\[
\begin{aligned}
U(k,\ell)
&=\sum_{a=k}^{\ell}\frac{x_a-x_k}{x_\ell-x_k}\ge\sum_{a=k'}^{\ell}\frac{x_a-x_k}{x_\ell-x_k}\\
&\ge\sum_{a=k'}^{\ell}\frac{x_a-x_{k'}}{x_\ell-x_{k'}}=U(k',\ell),
\end{aligned}
\]
where the second inequality follows from
\[
\frac{x_a-x_k}{x_\ell-x_k}
-
\frac{x_a-x_{k'}}{x_\ell-x_{k'}}
=
\frac{(x_{k'}-x_k)(x_\ell-x_a)}
     {(x_\ell-x_k)(x_\ell-x_{k'})}
\ge 0 .
\]

Next, $D(\ell,j)$ is monotone increasing in $j$. Indeed, for
$\ell<j<j'$,
\[
D(\ell,j')
-
D(\ell,j)
\ge
\sum_{a=\ell}^{j}
\left(
\frac{x_{j'}-x_a}{x_{j'}-x_\ell}
-
\frac{x_j-x_a}{x_j-x_\ell}
\right),
\]
and each summand is nonnegative since
\[
\frac{x_{j'}-x_a}{x_{j'}-x_\ell}
-
\frac{x_j-x_a}{x_j-x_\ell}
=
\frac{(x_{j'}-x_j)(x_a-x_\ell)}
     {(x_{j'}-x_\ell)(x_j-x_\ell)}
\ge 0 .
\]
The additional terms in $D(\ell,j')$ with indices $a>j$ are also
nonnegative. Hence $D(\ell,j')\ge D(\ell,j)$.

Because $\rho_\ell$ reverses the order of the columns, the function
\[
\widetilde m(r)=D(\ell,\rho_\ell(r))
\]
is monotone decreasing in $r$. Thus both $v$ and $\widetilde m$ are
monotone decreasing, and $g$ is concave. By Lemma~1,
\[
\widetilde G(k,r)
=
g\bigl(v(k)+\widetilde m(r)\bigr)
=
\lambda h\!\left(
\frac{U(k,\ell)+D(\ell,\rho_\ell(r))}{d}
\right)
\]
satisfies the quadrangle inequality. Since
\[
\widetilde M_{k,r}=A(k)+\widetilde G(k,r),
\]
and the row terms $A(k)$ cancel from the two sides of the quadrangle
inequality, $\widetilde M$ also satisfies \mbox{the quadrangle inequality.}
\end{proof}

\section{Proof of~\Cref{lem:Hirschberg_opt}}\label{app:Hirschberg_opt}
We restate the lemma for convenience. 
\Hirschbergopt*

\newcommand{\dopt}{\mathrm{dp}_{\mathrm{OPT}}}
\newcommand{\doptB}{\mathrm{dp}^{B}_{\mathrm{OPT}}}
\newcommand{\Cost}{\mathrm{C}}

\begin{proof}

Consider a subproblem with $n$ candidate indices and $r$ quantization values to reconstruct. A forward DP layer has one entry for each ordered pair $(\ell,j)$ of consecutive relevant boundary indices, so it contains $O(n^2)$ entries. By Lemma 2, for every fixed $\ell$ the corresponding transition matrix is Monge, and SMAWK computes the minimizing predecessor for all choices of $j$ in $O(n)$ time. Since there are $O(n)$ choices of $\ell$, one whole layer is computed in $O(n^2)$ time. The recurrence for layer $i$ depends only on layer $i-1$, so a forward run for $r$ layers uses $O(n^2 r)$ time and only two $O(n^2)$ arrays. The backward recurrence is symmetric and has \mbox{the same bounds.}

At a Hirschberg step, set $t=\lfloor r/2\rfloor$. We run the forward recurrence for $t+1$ layers and the backward recurrence for $r-t+1$ layers, retaining only the current and previous layer in each run. We then scan the final forward and backward layers and choose a pair $(\ell^*,j^*)$ minimizing
\[
  \dopt(t+1,\ell,j)
  + \doptB(r-t+1,\ell,j)
  - \Cost[\ell,j].
\]
For a fixed pair $(\ell,j)$, the first term is the optimal cost of the left partial solution ending at $\ell$ with $j$ as its right neighbor, and the second term is the optimal cost of the right partial solution starting at $j$ with $\ell$ as its left neighbor. The segment cost $\Cost[\ell,j]$ is present in both partial costs, so it is subtracted once. The entropy contributions at the two crossing quantization values are also well-defined: the left pass charges the quantization value at $\ell$ using its left neighbor and $j$, while the right pass charges the quantization value at $j$ using $\ell$ and its right neighbor. Therefore the displayed expression is exactly the best total cost among solutions whose crossing adjacent pair is $(\ell,j)$. Minimizing over all pairs recovers the crossing pair of a globally optimal solution; otherwise, concatenating the two partial solutions for a better pair would give a strictly better \mbox{feasible full solution.}

After $(\ell^*,j^*)$ is fixed, no remaining quantization value can lie strictly between $\ell^*$ and $j^*$. The reconstruction therefore decomposes into two independent recursive subproblems: the left side, which reconstructs the quantization values up to $\ell^*$ with $j^*$ fixed as the outside neighbor, and the right side, which reconstructs the quantization values from $j^*$ onward with $\ell^*$ fixed as the outside neighbor. The two sets of unfixed candidate indices are disjoint. Endpoint convention choices can add only $O(1)$ fixed boundary indices to a subproblem and do not affect \mbox{the asymptotic bounds.}

The space bound follows immediately. During any forward or backward pass, we store only a constant number of $O(n^2)$ layers. The largest subproblem has $n\le d$, so these layers occupy $O(d^2)$ space. The precomputed arrays $U,D$, and $\Cost$ also occupy $O(d^2)$ space. Once a crossing pair is found, we store only that pair and recurse into one side; when that recursive call returns, the same work arrays are reused for the other side. The recursion stack and the output quantization values use at most $O(s)\le O(d)$ additional space, which is dominated by $O(d^2)$. Thus the peak \mbox{space is $O(d^2)$.}

It remains to bound the running time. For a subproblem with parameters $(n,r)$, the two half-runs and the final scan cost
\[
  O(n^2(t+1)) + O(n^2(r-t+1)) + O(n^2) = O(n^2 r).
\]
At recursion depth $q$, each active subproblem has at most $\lceil s/2^q\rceil+O(1)$ quantization values remaining. Moreover, the unfixed candidate sets of the active subproblems are disjoint subsets of the original $d$ indices, so
\[
  \sum_{u} n_u^2 \le \left(\sum_{u} n_u\right)^2 \le d^2,
\]
where $u$ ranges over the subproblems at that depth and $n_u$ is the number of entries in the subproblem. Hence the total work at depth $q$ is at most
\[
  O\!\left(\left(\frac{s}{2^q}+1\right)d^2\right).
\]
Summing over all $O(\log s)$ recursion depths gives
\[
  \sum_{q\ge 0} O\!\left(\left(\frac{s}{2^q}+1\right)d^2\right)
  = O(d^2s) + O(d^2\log s)
  = O(d^2s).
\]
Therefore, Hirschberg reduces the memory to $O(d^2)$ while preserving the \mbox{$O(d^2\cdot s)$ time complexity.}
\end{proof}

\section{Generalizing to Arbitrary $P$}
\label{app:general-p}

We now remove the simplifying assumption that $P=X$. Let
$P=\langle p_1,\ldots,p_p\rangle$ be the sorted set of permissible quantization
values. We assume that $P$ contains at least one value not larger than $x_1$ and
at least one value not smaller than $x_d$; otherwise, no unbiased quantizer with
values in $P$ can represent all \mbox{entries of $X$.}

For every $1\le a<b\le p$, define \[
X_{a,b}=\{x\in X\mid p_a<x\le p_b\}.
\] Thus, entries equal to the left endpoint are excluded from $U_P[a,b]$ and $D_P[a,b]$ for every interval. For the first selected quantization value, its deterministic mass is added explicitly in the initialization below. This convention is also exactly the one implemented by the prefix differences
$N(p_b)-N(p_a)$, $S(p_b)-S(p_a)$, and $S_2(p_b)-S_2(p_a)$. The MSE contribution
is unaffected because entries equal to a quantization value have zero variance.
We define
\[
C_P[a,b] =
\sum_{x\in X_{a,b}} (p_b-x)(x-p_a),
\]
and similarly
{\small
\[
U_P[a,b]=\sum_{x\in X_{a,b}} \frac{x-p_a}{p_b-p_a},
\qquad
D_P[a,b]=\sum_{x\in X_{a,b}} \frac{p_b-x}{p_b-p_a}.
\]
}
We further write $U_P[n,n]=D_P[n,n]$ for all $n$.
Thus, if $p_k,p_\ell,p_j$ are three consecutive quantization values, the
expected frequency of $p_\ell$ is
\[
U_P[k,\ell] + D_P[\ell,j].
\]

Since $P$ may contain values outside the range of $X$, the first and last
quantization values are no longer fixed to $x_1$ and $x_d$. Let
$L=\{a\in[p]\mid p_a\le x_1\}$ be the feasible choices for the first
quantization value, and let $R=\{b\in[p]\mid p_b\ge x_d\}$ be the feasible
choices for the last one. For $a\in L$ and $a<j$, we initialize
\[
dp^P_{\mathrm{OPT}}(2,a,j)
=
C_P[a,j]
+
\lambda\cdot
h\!\left(
\frac{n_a + D_P[a,j]}{d}
\right),
\]
where $n_a = |\{x\in X \mid x=p_a\}|$. This state corresponds to using
$p_a$ and $p_j$ as the two current boundary values, while charging the entropy
contribution of the left boundary $p_a$ and leaving the contribution of $p_j$
to \mbox{be charged later.}

For $i\ge 3$, the recurrence is identical to the one above after replacing
input indices by permissible-value indices:
{\scriptsize
\[
\begin{aligned}
dp^P_{\mathrm{OPT}}(i,\ell,j)
={}&C_P[\ell,j]+\min_{k<\ell}\Bigl\{\\[-0.5ex]
&dp^P_{\mathrm{OPT}}(i-1,k,\ell)+\lambda\cdot h\!\left(
\frac{U_P[k,\ell]+D_P[\ell,j]}{d}
\right)\Bigr\}.
\end{aligned}
\]
}
Finally, the optimal value using at most $s$ quantization values is
\[
\min_{2\le i\le s}
\min_{\substack{\ell<j\\ j\in R}}
\left\{
dp^P_{\mathrm{OPT}}(i,\ell,j)
+
\lambda\cdot
h\!\left(
\frac{U_P[\ell,j]}{d}
\right)
\right\}.
\]
If exactly $s$ values are required, the outer minimum over $i$ is \mbox{replaced by
$i=s$.}

The quantities $C_P,U_P,D_P$ can be computed efficiently using prefix sums over
the sorted input. Let $N(t)$, $S(t)$, and $S_2(t)$ denote the number, sum, and
sum of squares of entries of $X$ that are at most $t$. Then for every $a<b$,
letting
{\small
\[
\begin{aligned}
N_{a,b}&=N(p_b)-N(p_a),\qquad
S_{a,b}=S(p_b)-S(p_a),\\
S^{(2)}_{a,b}&=S_2(p_b)-S_2(p_a),
\end{aligned}
\]
}
we have
{\small
\[
U_P[a,b]=\frac{S_{a,b}-p_aN_{a,b}}{p_b-p_a},
\qquad
D_P[a,b]=\frac{p_bN_{a,b}-S_{a,b}}{p_b-p_a},
\]
}
and
\[
C_P[a,b]
=
(p_a+p_b)S_{a,b}
-
S^{(2)}_{a,b}
-
p_ap_bN_{a,b}.
\]
After a linear scan over $X$ and $P$ to build these prefix quantities, all
pairwise costs are obtained in $O(p^2+d)$ time \mbox{and $O(p^2)$ space.}

The SMAWK acceleration from~\Cref{sec:smawk} continues to apply. The proof of~\Cref{lem:quadrangleM} only uses the ordering of the candidate values and the monotonicity of
the corresponding round-up and round-down sums; these arguments are unchanged
when the candidates are $p_1,\ldots,p_p$ rather than $x_1,\ldots,x_d$.
Likewise, the Hirschberg space reduction from~\Cref{sec:hirschberg} applies verbatim
after replacing $d$ DP indices by $p$ permissible-value indices. Therefore, the
general-$P$ optimal algorithm runs in
$
O(p^2\cdot s + d)
$
time and uses $O(p^2+d)$ space. When $P=X$ and the boundary values are fixed,
this reduces to the \mbox{formulation of~\Cref{lem:Hirschberg_opt}.}

\section{Details for Optimality Through Two-Way Time-Sharing}
\label{app:timesharing-details}

\subsection{From Lagrange Multipliers to Entropy Constraints}
\label{app:timesharing-lagrange}
The Lagrangian formulation need not find the optimal single quantizer for a
given entropy constraint. For example, Let $d$ be divisible by $2$, and consider
\[
\begin{aligned}
X&=\bigl(\underbrace{0,\ldots,0}_{d/2\text{ times}},
\underbrace{2,\ldots,2}_{d/2\text{ times}}\bigr),\\
P&=(0,1,2,3,4),\qquad s=2.
\end{aligned}
\]
The feasible quantizers are:
\[
\begin{array}{c|c|c|c}
\text{Quantizer $Q$} & \text{Output distribution} & H(\widehat X) & \mathit{vNMSE}\\
\hline
A=\{0,4\} & (3/4,\,1/4) & 0.811 & 1\\
B=\{0,3\} & (2/3,\,1/3) & 0.918 & 1/2\\
C=\{0,2\} & (1/2,\,1/2) & 1.000 & 0
\end{array}
\]
If $b=0.95$, then $B$ is the hard-constrained optimum. However, $B$ beats
$C$ only if
\[
\lambda \ge \frac{1/2}{1-0.918296}> 6,
\]
while $B$ beats $A$ only if
\[
\lambda \le \frac{1/2}{0.918296-0.811278}< 5.
\]
Thus, no value of $\lambda$ produces $B$ as \mbox{the Lagrangian optimum.}

The standard approach is to \emph{time-share} solutions that are optimal for
some value of $\lambda$. Here, the quantizer can pick a random subset of
$d\cdot\alpha$ entries, where
$\alpha=\frac{0.95-0.811}{1-0.811}\approx0.74$, encode them with $C$, and
encode the rest with $A$.\footnote{By using shared randomness to pick the
subset, the indices need not be communicated and the dequantizer can
\mbox{independently generate them.}} Since the entropy and vNMSE are linear in the
time-sharing weight, the result has entropy
$H(\widehat X)=0.811\cdot\alpha+1\cdot(1-\alpha)=0.95$ and vNMSE
$0\cdot\alpha+1\cdot(1-\alpha)\approx0.26 < 1/2$ , improving over 
\mbox{quantizer $B$.}

\subsection{Quantizing to Non-Adjacent Values in $Q$}
\label{app:timesharing-nonadjacent}
In the main development, each entry $x\in X$ is stochastically quantized
between
$a_x=\max\set{q\in Q\mid q\le x}$ and
$b_x=\min\set{q\in Q\mid q\ge x}$.
This intuitive restriction is not guaranteed to give the optimal single
quantizer. Let $d$ be divisible by $3$ and consider
\[
\begin{aligned}
X&=\bigl(\underbrace{0,\ldots,0}_{d/3\text{ times}},
\underbrace{1,\ldots,1}_{d/3\text{ times}},
\underbrace{2,\ldots,2}_{d/3\text{ times}}\bigr),\\
P&=(0,1,2),\qquad s=3,\qquad b=\frac32.
\end{aligned}
\]
Under the adjacent-rounding restriction, any feasible quantizer must contain
$0$ and $2$. Hence, the only relevant choices are
\[
Q=\set{0,2}
\qquad\text{and}\qquad
Q=\set{0,1,2}.
\]
The quantizer $Q=\set{0,1,2}$ maps all entries deterministically and yields
output frequencies
\[
\left(\frac13,\frac13,\frac13\right),
\]
so
\[
H(\widehat X)=\log_2 3>\frac32.
\]
Therefore, the entropy constraint forces the adjacent-rounding solution to use
\[
Q=\set{0,2}.
\]
This gives output frequencies $(1/2,1/2)$, so its entropy is $1\le b$. Its MSE
is $d/3$. Since
\[
\norm{X}^2=\frac d3\cdot1^2+\frac d3\cdot2^2=\frac{5d}{3},
\]
its vNMSE is
\[
\frac{d/3}{5d/3}=\frac15.
\]

If we instead allow $Q=\set{0,1,2}$ and permit different entries with the same value to be quantized differently, we can map half of the entries equal
to $1$ deterministically to $1$, and stochastically quantize the other half
between the non-adjacent values $\set{0,2}$, with probability $1/2$ on each
endpoint. This remains unbiased entrywise. The expected output frequencies are
\[
\left(\frac{5}{12},\frac16,\frac{5}{12}\right),
\]
and therefore
\[
H(\widehat X)
=
-2\cdot\frac{5}{12}\log_2\frac{5}{12}
-\frac16\log_2\frac16
\approx1.48336
<\frac32.
\]
The MSE is $d/6$, and hence the vNMSE is
\[
\frac{d/6}{5d/3}=\frac1{10}.
\]
Thus, allowing non-adjacent quantization and allowing two identical entries to
be treated differently can strictly improve the vNMSE while satisfying the
same entropy budget when a single quantizer \mbox{must be used.}

\subsection{Proof of Proposition~\ref{prop:adjacent-time-sharing-dominates}}
\label{app:timesharing-proof}
Under two-way time-sharing, allowing non-adjacent unbiased rounding does not
improve the entropy--vNMSE frontier, even when identical input entries may be
treated differently. We prove the stronger statement that every two-way
time-share using arbitrary unbiased rounding is dominated by a two-way
time-share using ordinary \mbox{adjacent stochastic quantization.}

Let $X=(x_1,\ldots,x_d)\in\mathbb{R}^d$, and assume $\norm{X}>0$. For a finite
quantizer $Q\subset\mathbb{R}$, an arbitrary unbiased rounding rule is
specified by probabilities
\[
K_i(q)=\Pr[\widehat{x}_i=q],
\qquad i\in[d],\ q\in Q,
\]
satisfying
\[
K_i(q)\ge 0,
\qquad
\sum_{q\in Q}K_i(q)=1,
\qquad
\sum_{q\in Q}qK_i(q)=x_i.
\]
The rule may depend on the coordinate $i$, so equal input values need not be
treated identically. Define its expected output-frequency vector by
\[
\pi_K(q)=\frac{1}{d}\sum_{i=1}^d K_i(q),
\qquad q\in Q.
\]
Thus, $\pi_K(q)$ is the expected fraction of coordinates emitted as $q$, and
$\sum_q\pi_K(q)=1$. Its entropy and vNMSE are
\[
R(K)=H(\pi_K)
=-\sum_{q\in Q}\pi_K(q)\log_2\pi_K(q)
\]and\[
D(K)=
\frac{1}{\norm{X}^2}
\sum_{i=1}^d\sum_{q\in Q}K_i(q)(q-x_i)^2,
\]where $0\log_2 0$ is defined as $0$.

A two-way time-share consists of rules $K_1,K_2$ for all $d$
coordinates and a weight $\alpha\in[0,1]$. The shared selector is independent
of the data values; equivalently, one may time-share by partitioning $X$ and using separate quantizers for each part. The quantizer used for a given entry need not be encoded, and the two parts may be
entropy-coded separately. This is the convexified time-sharing convention used
in the preceding discussion. Its rate and distortion are therefore
\begin{align*}
R_{\mathrm{ts}}
&=\alpha R(K_1)+(1-\alpha)R(K_2),\\
D_{\mathrm{ts}}
&=\alpha D(K_1)+(1-\alpha)D(K_2).
\end{align*}
Ordinary adjacent stochastic quantization with a sorted set
$S=\{s_1<\cdots<s_m\}$ maps each $x_i$ to the two consecutive values of $S$
that encompass it, with the unique probabilities that make the output mean
\mbox{equal to $x_i$.}

\adjacenttimesharingdominates*

The proof uses \mbox{three elementary facts.}

\begin{lemma}[The expected output frequencies determine the vNMSE]
\label{lem:timesharing-frequencies-determine-mse}
For every unbiased rule $K$ into $Q$,
\[
D(K)=
\frac{d\sum_{q\in Q}\pi_K(q)q^2-\norm{X}^2}{\norm{X}^2}.
\]Consequently, for fixed $X$ and fixed output values, two unbiased rules with
the same expected output-frequency vector have exactly the same vNMSE.
\end{lemma}

\begin{proof}
For each coordinate $i$, unbiasedness gives
\ifarxiv
\begin{align*}
\E[(\widehat{x}_i-x_i)^2]
&=\sum_{q\in Q}K_i(q)(q-x_i)^2\\
&=\sum_{q\in Q}K_i(q)q^2
  -2x_i\sum_{q\in Q}qK_i(q)
  +x_i^2\sum_{q\in Q}K_i(q)\\
&=\sum_{q\in Q}K_i(q)q^2-x_i^2.
\end{align*}
\else
\begin{align*}
\E[(\widehat{x}_i-x_i)^2]
&=\sum_{q\in Q}K_i(q)(q-x_i)^2\\
&=\sum_{q\in Q}K_i(q)q^2
  -2x_i\sum_{q\in Q}qK_i(q)\\
&\quad+x_i^2\sum_{q\in Q}K_i(q)\\
&=\sum_{q\in Q}K_i(q)q^2-x_i^2.
\end{align*}
\fi
Summing over $i$, exchanging the two finite sums, and dividing by
$\norm{X}^2$ yields
\begin{align*}
D(K)&=\frac{\sum_{q\in Q}q^2\sum_{i=1}^d K_i(q)-\sum_{i=1}^d x_i^2}{\norm{X}^2}\\&=\frac{
d\sum_{q\in Q}\pi_K(q)q^2-\norm{X}^2}{\norm{X}^2}.
\end{align*}
The right-hand side depends on $K$ only through $\pi_K$.
\end{proof}

\begin{lemma}[Every unbiased frequency vector is a mixture of adjacent-SQ
frequency vectors]\label{lem:timesharing-frequency-convex-hull}
Fix a feasible quantizer $Q=\{q_1<\cdots<q_m\}$, meaning
$q_1\le\min_i x_i$ and $q_m\ge\max_i x_i$. Let $\mathcal{F}_Q$ be the set of
expected output-frequency vectors produced by arbitrary unbiased rules into
$Q$. Let $\mathcal{A}_Q$ be the set of expected output-frequency vectors
produced by ordinary adjacent stochastic quantization using any feasible
subset $S\subseteq Q$. Then
\[
\mathcal{F}_Q=\operatorname{conv}(\mathcal{A}_Q).
\]
\end{lemma}

\begin{proof}
If $m=1$, feasibility implies that $x_i=q_1$ for every $i$, and both
$\mathcal{F}_Q$ and $\mathcal{A}_Q$ contain the same single frequency vector.
Assume henceforth that $m\ge2$.

Every adjacent-SQ rule is an unbiased rule into $Q$, after assigning zero
probability to values in $Q\setminus S$. Hence,
$\operatorname{conv}(\mathcal{A}_Q)\subseteq\mathcal{F}_Q$, because
\mbox{$\mathcal{F}_Q$ is convex.}

For the reverse inclusion, fix arbitrary real numbers $c(q)$ for $q\in Q$.
We first minimize the linear quantity
\[
\langle c,\pi_K\rangle
=\sum_{q\in Q}c(q)\pi_K(q)
=\frac{1}{d}\sum_{i=1}^d\sum_{q\in Q}c(q)K_i(q)
\]
over all unbiased rules $K$. The constraints on different coordinates are
independent, so this minimization \mbox{separates over $i$.}

Let $g$ be the lower convex envelope of the finitely many points
\[
(q_1,c(q_1)),\ldots,(q_m,c(q_m)).
\]
Equivalently, $g$ is the largest convex function on $[q_1,q_m]$ satisfying
$g(q)\le c(q)$ for every $q\in Q$. For any distribution $p$ on $Q$ with mean
$x$, Jensen's inequality gives
\[
g(x)
=g\!\left(\sum_{q\in Q}p(q)q\right)
\le\sum_{q\in Q}p(q)g(q)
\le\sum_{q\in Q}p(q)c(q).
\]
Thus, $g(x)$ is a lower bound on the cost of every unbiased distribution \mbox{with
mean $x$.}

Let $S=\{v_1<\cdots<v_k\}\subseteq Q$ be the $q$-coordinates of the vertices
on the lower convex envelope. The endpoints $q_1,q_m$ belong to $S$, so $S$ is
feasible. On every interval $[v_j,v_{j+1}]$, the function $g$ is the line
segment joining $(v_j,c(v_j))$ and $(v_{j+1},c(v_{j+1}))$. Therefore, if
$v_j\le x\le v_{j+1}$, the distribution
\begin{align*}
\Pr[Y=v_j]
&=\frac{v_{j+1}-x}{v_{j+1}-v_j},\\
\Pr[Y=v_{j+1}]
&=\frac{x-v_j}{v_{j+1}-v_j}.
\end{align*}
has mean $x$ and expected cost exactly $g(x)$. This is precisely ordinary
adjacent stochastic quantization with the set $S$. Applying it to every
coordinate attains the minimum of $\langle c,\pi_K\rangle$. Hence, for every
$c$,
\begin{equation}
\min_{\pi\in\mathcal{F}_Q}\langle c,\pi\rangle
=\min_{a\in\mathcal{A}_Q}\langle c,a\rangle.
\label{eq:timesharing-linear-minima}
\end{equation}

It remains to show that~\eqref{eq:timesharing-linear-minima} implies
$\mathcal{F}_Q\subseteq\operatorname{conv}(\mathcal{A}_Q)$. Suppose instead
that some $\pi\in\mathcal{F}_Q$ lies outside the compact convex set
$C=\operatorname{conv}(\mathcal{A}_Q)$. Choose $y\in C$ minimizing
$\norm{\pi-y}$ and set $c=y-\pi$. For every $a\in C$, convexity of $C$ implies
that $y+t(a-y)\in C$ for $t\in[0,1]$. Since $t=0$ minimizes
\[
\norm{\pi-(y+t(a-y))}^2,
\]
the right derivative at $0$ is nonnegative, which gives
\[
(y-\pi)\cdot(a-y)\ge0.
\]
Thus, $c\cdot a\ge c\cdot y$ for every $a\in C$, while
\[
c\cdot\pi=c\cdot y-\norm{c}^2<c\cdot y.
\]
Therefore,
\begin{align*}
\min_{f\in\mathcal{F}_Q}c\cdot f
&\le c\cdot\pi
<\min_{a\in C}c\cdot a\\
&=\min_{a\in\mathcal{A}_Q}c\cdot a,
\end{align*}
contradicting~\eqref{eq:timesharing-linear-minima}. Hence,
$\mathcal{F}_Q=\operatorname{conv}(\mathcal{A}_Q)$.
\end{proof}

\begin{lemma}[A rate-constrained mixture needs at most two quantizers]
\label{lem:timesharing-two-quantizers-suffice}
Let $(R_j,D_j)$, $j=1,\ldots,N$, be finitely many rate--vNMSE pairs. If some
mixture of these pairs has rate at most $b$, then the minimum vNMSE among all
mixtures with rate at most $b$ is attained by a mixture supported on at most two pairs.
\end{lemma}

\begin{proof}
Consider the finite-dimensional linear program
\[
\min_{w_1,\ldots,w_N}\sum_{j=1}^N w_jD_j
\]
subject to
\begin{equation}
\sum_{j=1}^N w_j=1,
\qquad
\sum_{j=1}^N w_jR_j\le b,
\qquad
w_j\ge0.
\label{eq:timesharing-lp}
\end{equation}
Its feasible set is compact, so an optimum exists. An optimum may be chosen
to be an extreme point of the feasible set. Indeed, if an optimal point is a
nontrivial convex combination of two feasible points, linearity implies that
both points are also optimal; repeating within the resulting
lower-dimensional face reaches an \mbox{extreme optimal point.}

Let $w$ be an extreme feasible point, and suppose that it has at least three
positive coordinates. The vectors $(1,R_j)\in\mathbb{R}^2$ corresponding to
those positive coordinates are linearly dependent. Hence, there is a nonzero
vector $z$, supported only on those coordinates, such that
\[
\sum_jz_j=0,
\qquad
\sum_jz_jR_j=0.
\]
For sufficiently small $\varepsilon>0$, both $w+\varepsilon z$ and
$w-\varepsilon z$ are nonnegative. They preserve both sums
in~\eqref{eq:timesharing-lp}, so they are distinct feasible points, and
\[
w=\frac12(w+\varepsilon z)+\frac12(w-\varepsilon z).
\]
This contradicts extremality. Therefore, an extreme feasible point, and in
particular an extreme optimum, has at most \mbox{two positive coordinates.}
\end{proof}

\begin{proof}[Proof of Proposition~\ref{prop:adjacent-time-sharing-dominates}]
The existence of an unbiased rule $K_t$ into $Q_t$ implies that every
$x_i$ lies in the convex hull of $Q_t$; hence $Q_t$ is feasible. Apply Lemma~\ref{lem:timesharing-frequency-convex-hull} separately to $K_1$ and
$K_2$. For each $t\in\{1,2\}$, there are feasible subsets $S_{t,r}\subseteq Q_t$ and weights $\theta_{t,r}\ge0$ with
$\sum_r\theta_{t,r}=1$ such that
\begin{equation}
\pi_{K_t}=\sum_r\theta_{t,r}\pi_{S_{t,r}}.
\label{eq:timesharing-frequency-decomposition}
\end{equation}
By Lemma~\ref{lem:timesharing-frequencies-determine-mse} and the linearity of
its right-hand side in $\pi$,
\begin{equation}
D(K_t)=\sum_r\theta_{t,r}D(S_{t,r}).
\label{eq:timesharing-distortion-decomposition}
\end{equation}
Entropy is concave. To see this directly, the function
$h(u)=-u\log_2u$, with $h(0)=0$, satisfies
$h''(u)=-1/(u\ln2)<0$ for $u>0$. Therefore,
Equation~\eqref{eq:timesharing-frequency-decomposition} gives
\begin{equation}
R(K_t)=H(\pi_{K_t})
\ge\sum_r\theta_{t,r}H(\pi_{S_{t,r}})
=\sum_r\theta_{t,r}R(S_{t,r}).
\label{eq:timesharing-entropy-decomposition}
\end{equation}

Combining the decompositions for the two original arms with weights
\[
w_{1,r}=\alpha\theta_{1,r},
\qquad
w_{2,r}=(1-\alpha)\theta_{2,r}
\]
produces a finite mixture of adjacent-SQ quantizers whose MSE equals the
original two-way vNMSE by
Equation~\eqref{eq:timesharing-distortion-decomposition}, and whose rate is no
larger than the original two-way rate by
Equation~\eqref{eq:timesharing-entropy-decomposition}. Every quantizer in this
finite collection is a subset of $Q_1$ or $Q_2$, so it has at \mbox{most $s$ values.}

This finite mixture is only an intermediate existence certificate. Apply
Lemma~\ref{lem:timesharing-two-quantizers-suffice} to its finitely many
adjacent-SQ rate--vNMSE pairs, using the original rate $R_{\mathrm{ts}}$ as the
budget. The resulting optimal mixture uses at most two adjacent-SQ quantizers,
has rate at most $R_{\mathrm{ts}}$, and has vNMSE no larger than the intermediate
mixture, hence no larger than $D_{\mathrm{ts}}$. This \mbox{proves the proposition.}
\end{proof}

Consequently, there is no instance in this model where a
two-way time-share using non-adjacent unbiased rounding has strictly smaller
vNMSE than every adjacent-only solution that may time-share between at \mbox{most two
quantizers.}

For the example above, adjacent time-sharing between $\{0,1,2\}$ and
$\{0,2\}$ is already better than the proposed non-adjacent rule. Put weight
\[
\alpha=\frac{3/2-1}{\log_2 3-1}\approx0.855
\]
on $\{0,1,2\}$ and the remaining weight on $\{0,2\}$. The average entropy is
exactly $3/2$, while the vNMSE is
\[
(1-\alpha)\frac15\approx0.029<\frac1{10}.
\]
Thus, the existing example does not separate non-adjacent rounding from
adjacent rounding once two-way \mbox{time-sharing is admitted.}

\ifarxiv\else\onecolumn\fi
\section{Evaluation Details}\label{app:evaluation}

\setcounter{topnumber}{4}
\setcounter{bottomnumber}{4}
\setcounter{totalnumber}{8}
\setcounter{dbltopnumber}{4}
\renewcommand{\topfraction}{0.95}
\renewcommand{\bottomfraction}{0.95}
\renewcommand{\textfraction}{0.05}
\renewcommand{\floatpagefraction}{0.75}
\renewcommand{\dbltopfraction}{0.95}
\renewcommand{\dblfloatpagefraction}{0.75}

\newcommand{\evalfigurestart}{%
\begin{figure}[!htbp]}
\newcommand{\evalfigureend}{%
\end{figure}}

\evalfigurestart
    \centering
    \includegraphics[width=0.96\textwidth,height=0.78\textheight,
      keepaspectratio]{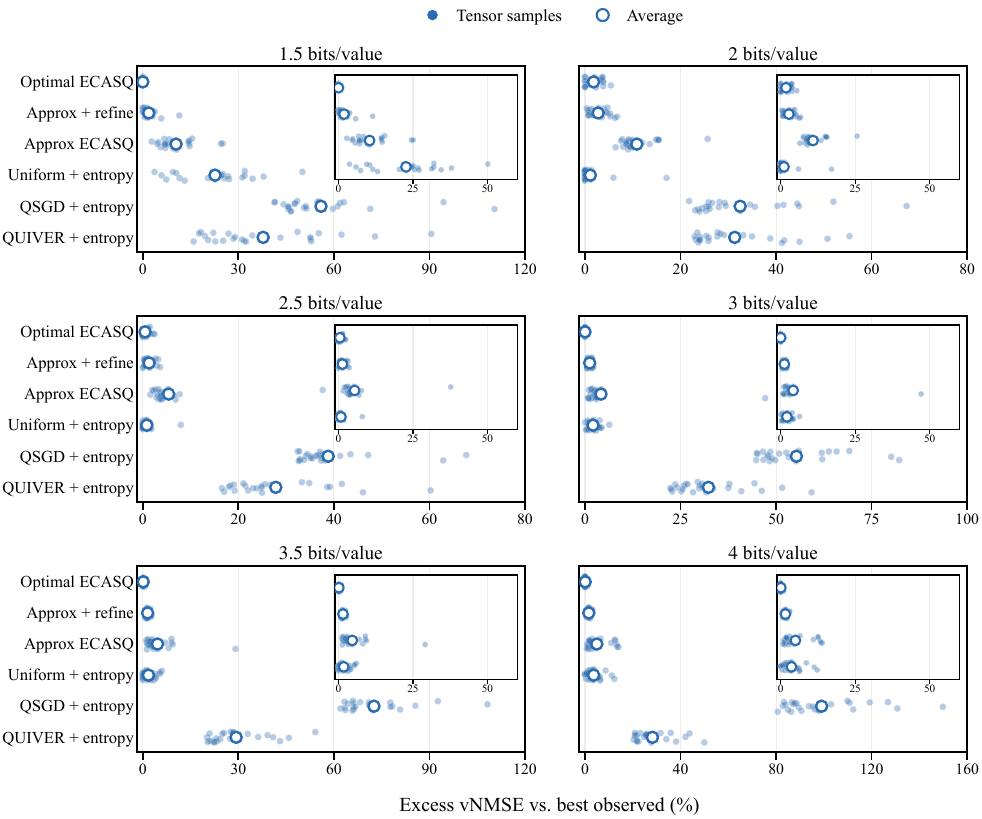}
    \caption{Matched-rate vNMSE over 25 BF16 real-model weight tensors.  Subfigures
    show $1.5$, $2$, $2.5$, $3$, $3.5$, and $4$ bits/value in reading order.
    Filled circles are individual tensors and the larger white-centered circle
    is their arithmetic mean.  Every full subfigure uses a linear horizontal axis
    sized to its observed range; its inset magnifies Optimal ECASQ, Approx ECASQ
    with refinement, Approx ECASQ, and uniform quantization with entropy coding
    on a common linear scale within this figure.  Values are relative to the best
    observed method for the same tensor and total rate, and no method is
    extrapolated.}
    \label{fig:eval-real-weights}
\evalfigureend

\evalfigurestart
    \centering
    \includegraphics[width=0.96\textwidth,height=0.78\textheight,
      keepaspectratio]{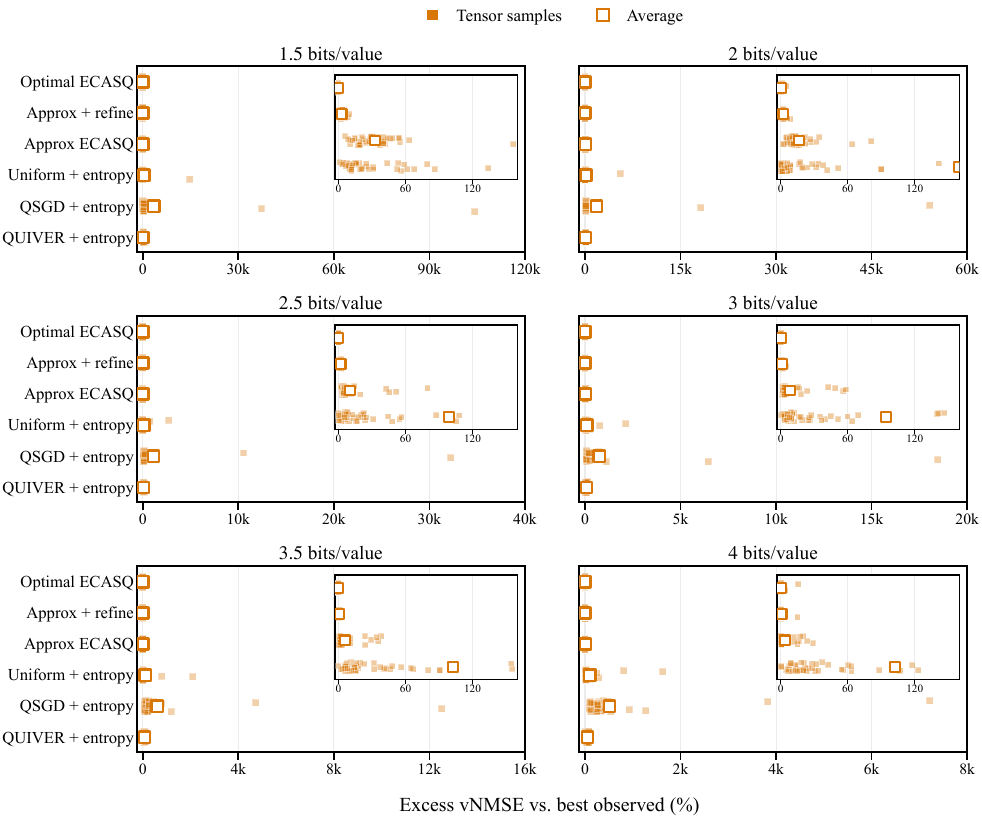}
    \caption{Matched-rate vNMSE over 45 BF16 real-model activation tensors.
    Subfigure order and axis conventions match \Cref{fig:eval-real-weights}.
    Filled squares are individual tensors and the larger white-centered square
    is their arithmetic mean.  Each inset uses a common linear scale within this
    figure and magnifies the three ECASQ variants together with uniform
    quantization with entropy coding.}
    \label{fig:eval-real-activations}
\evalfigureend

\evalfigurestart
    \centering
    \includegraphics[width=0.96\textwidth,height=0.78\textheight,
      keepaspectratio]{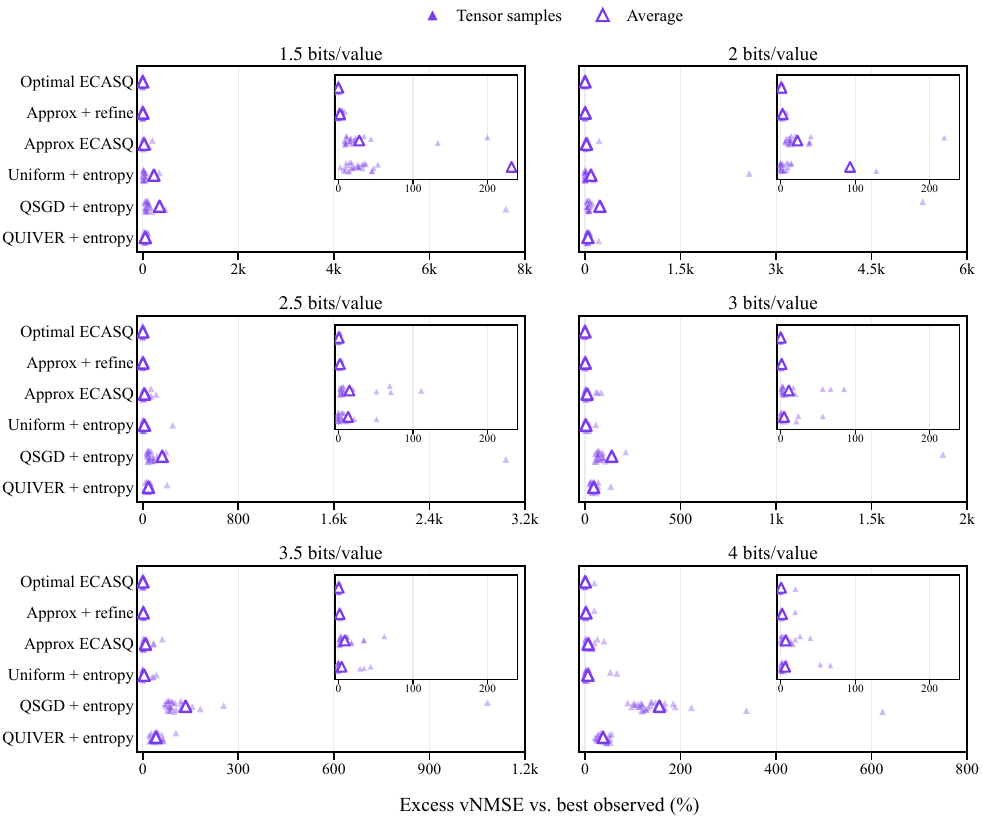}
    \caption{Matched-rate vNMSE over 30 BF16 real-model KV-cache tensors.
    Subfigure order and axis conventions match \Cref{fig:eval-real-weights}.
    Filled triangles are individual tensors and the larger white-centered
    triangle is their arithmetic mean.  Each inset uses a common linear scale
    within this figure and magnifies the three ECASQ variants together with
    uniform quantization with entropy coding.}
    \label{fig:eval-real-kv-cache}
\evalfigureend

\evalfigurestart
    \centering
    \includegraphics[width=0.98\textwidth]
      {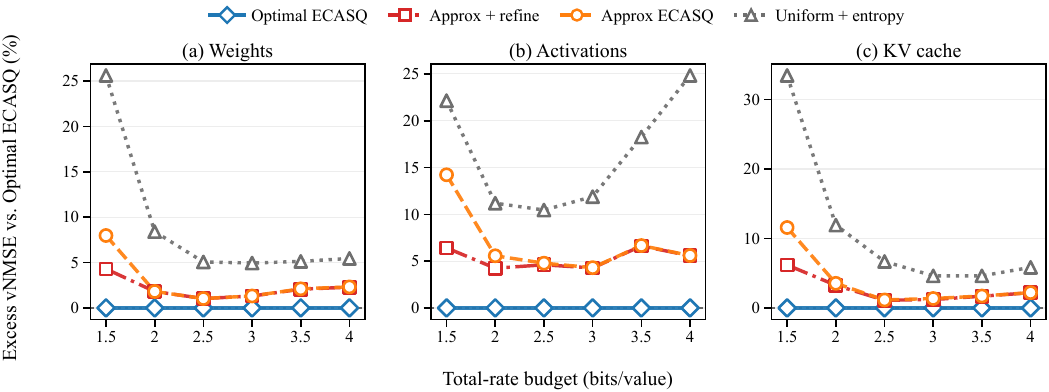}
    \caption{Matched-rate comparison on paired uniform permissible grids.
    The subfigures show weights, activations, and KV-cache tensors.  Curves report
    the arithmetic mean of each tensor's excess vNMSE relative to Optimal
    ECASQ at the same total-rate budget.  Approx ECASQ with refinement stays
    close to Optimal ECASQ even when all methods use the stored grids selected
    by the uniform baseline; lower is better.}
    \label{fig:eval-uniform-grid-control}
\evalfigureend

\subsection{Protocol and Rate Accounting}
All experiments use BF16 inputs and report the scale-free vNMSE defined in
\Cref{sec:preliminaries}.  To compare methods at a common storage budget, we include
both the ideal entropy-coded payload and the per-level metadata:
\[
R_{\mathrm{tot}}(Q,X)
  = H(\widehat X) + \frac{(16+16)\lvert Q\rvert}{d}
  = H(\widehat X) + \frac{32\lvert Q\rvert}{d}
  \quad\text{bits/value}.
\]
For each active quantization value, the first 16 bits store its BF16 reconstruction value
and the second 16 bits model its entropy descriptor, such as a quantized
frequency/CDF or code-length entry.  The ordering of $Q$ supplies the
symbol-to-value association, so we do not charge for a separate permutation.
This remains an ideal rate: $H(\widehat X)$ is the ideal payload rate, and we
exclude global framing and finite-stream redundancy.  At a target rate,
we average repeated rates arithmetically and vNMSEs geometrically, interpolate
vNMSE in log space between adjacent measured points, and never extrapolate.
Every reported method--tensor curve brackets all \mbox{six target rates.}

We compare Optimal ECASQ, Approx ECASQ, Approx ECASQ followed by five refinement
rounds, QUIVER with entropy coding, QSGD with entropy coding, and uniform
stochastic quantization with entropy coding.  The synthetic rate--distortion
study uses vectors of length $d=16{,}384$, five seeds, at most $s=64$ quantization values,
and Normal, Lognormal, Student-$t$, and sparse-outlier distributions.  We sweep
25 logarithmically spaced Lagrange multipliers from $10^{-6}$ through $10^2$
for ECASQ.  QUIVER uses requested codebook sizes from 2 through 256 in powers
of two; QSGD and uniform quantization continue through 4096.  We then
compare at \mbox{$R_{\mathrm{tot}}\in\{1.5,2,2.5,3,3.5,4\}$ bits/value.}

The real-model study contains 100 fixed samples from five model
families: Qwen2.5 0.5B and 1.5B, Gemma 3 1B IT, Llama 3.2 1B, and Phi-3.5 Mini.
It comprises 25 weight matrices, 45 activations, and 30 KV-cache tensors.  Each
sample contains $d=16{,}384$ entries.  Activations and KV caches come from the
first, middle, and last decoder blocks; the weights cover embedding,
attention-input, attention-output, MLP-input, and MLP-output families.  Optimal
ECASQ and Approx ECASQ ordinarily use $s=64$ and
$\lambda\in\{10^{-6},10^{-5},\ldots,10^2\}$; the baseline codebook-size grids
match the synthetic study.  To bracket every target without extrapolation, we
evaluate Optimal ECASQ on all 25 weight samples, add an $s=128$, $\lambda=0$
endpoint for one high-rate activation curve, extend fixed-codebook baselines
through 16,384 quantization values where needed, and use a 65,536-bin QUIVER approximation
grid for three extreme-range activations.  Each activation/KV calibration run
uses four 128-token sequences.  The plots use the lowest
measured/interpolated vNMSE for each tensor as a neutral ``best observed''
reference.  The deterministic real-model samples do not support \mbox{seed-wise
confidence intervals.}

\subsection{Matched-Rate vNMSE}
\Cref{fig:eval-real-weights,fig:eval-real-activations,fig:eval-real-kv-cache}
give the complementary real-model study at every target rate, with all 100
tensors present for every method.  Refined Approx ECASQ remains close to the
best observed method.  At 3 bits/value its average and maximum excess vNMSE are
$1.35\%$ and $4.8\%$, respectively.  Unrefined Approx ECASQ averages
$8.12\%$, lower than uniform, QUIVER, and QSGD at $44.35\%$, $54.9\%$, and
\mbox{$390.57\%$, respectively.}

\subsection{Paired Uniform-Grid Control}
\label{app:eval-uniform-grid-control}
To isolate the benefit of selecting $Q$ from a common permissible alphabet,
we first run uniform quantization on each real-model tensor.  At each target
rate, we retain the two measured uniform min--max grids $P$ that bracket the
target, then run all three ECASQ variants on those exact stored grids with
$s_{\max}=\lvert P\rvert$.  We convexify in
$(R_{\mathrm{tot}},\vnmse)$ space using linear-vNMSE time sharing; every
component comparison therefore has the same permissible grid and maximum
number of quantization values, and the rate includes the 32-bit per-level
metadata overhead.

Optimal ECASQ is limited to grids with at most 256 stored quantization values.  Applying
this limit as a single all-rate cohort leaves 25 weight tensors,
30 activation tensors, and 29 KV-cache tensors.  We omit the remaining
15 activation tensors and one KV-cache tensor from every method and target rather
than report partial coverage.

Across the 18 tensor-category--rate cells in
\Cref{fig:eval-uniform-grid-control}, refined Approx ECASQ, unrefined Approx
ECASQ, and uniform quantization have mean excess vNMSEs of $3.35\%$,
$4.41\%$, and $12.25\%$, respectively.  Refinement is within $6.62\%$ of
Optimal ECASQ in every cell, while the uniform baseline reaches
\mbox{$33.46\%$ excess.}

The synthetic study provides a controlled comparison in which Optimal ECASQ is
available throughout.  The main-paper \Cref{fig:eval-vnmse} selects the fully
covered 2, 3, and 4 bits/value targets.  Across those 12 distribution--rate
points, five refinement rounds reduce mean excess vNMSE from $4.97\%$ to
$1.37\%$, with a $4.06\%$ maximum on the five-seed aggregate curves.
Unrefined Approx ECASQ beats QSGD and QUIVER at all 12 points and uniform
quantization at 11 of 12.  Across all four distributions and six target rates,
refinement reduces the mean excess from $6.06\%$ to $1.4\%$ and the maximum
from $23.83\%$ to $4.06\%$ relative to Optimal ECASQ.
\Cref{fig:eval-synthetic-summary} shows all six matched-rate comparisons.  Ten
high-cardinality QUIVER runs failed; the no-extrapolation rule omits only
unsupported target points rather than \mbox{extending incomplete curves.}

\evalfigurestart
    \centering
    \includegraphics[width=0.96\textwidth,height=0.88\textheight,
    keepaspectratio]{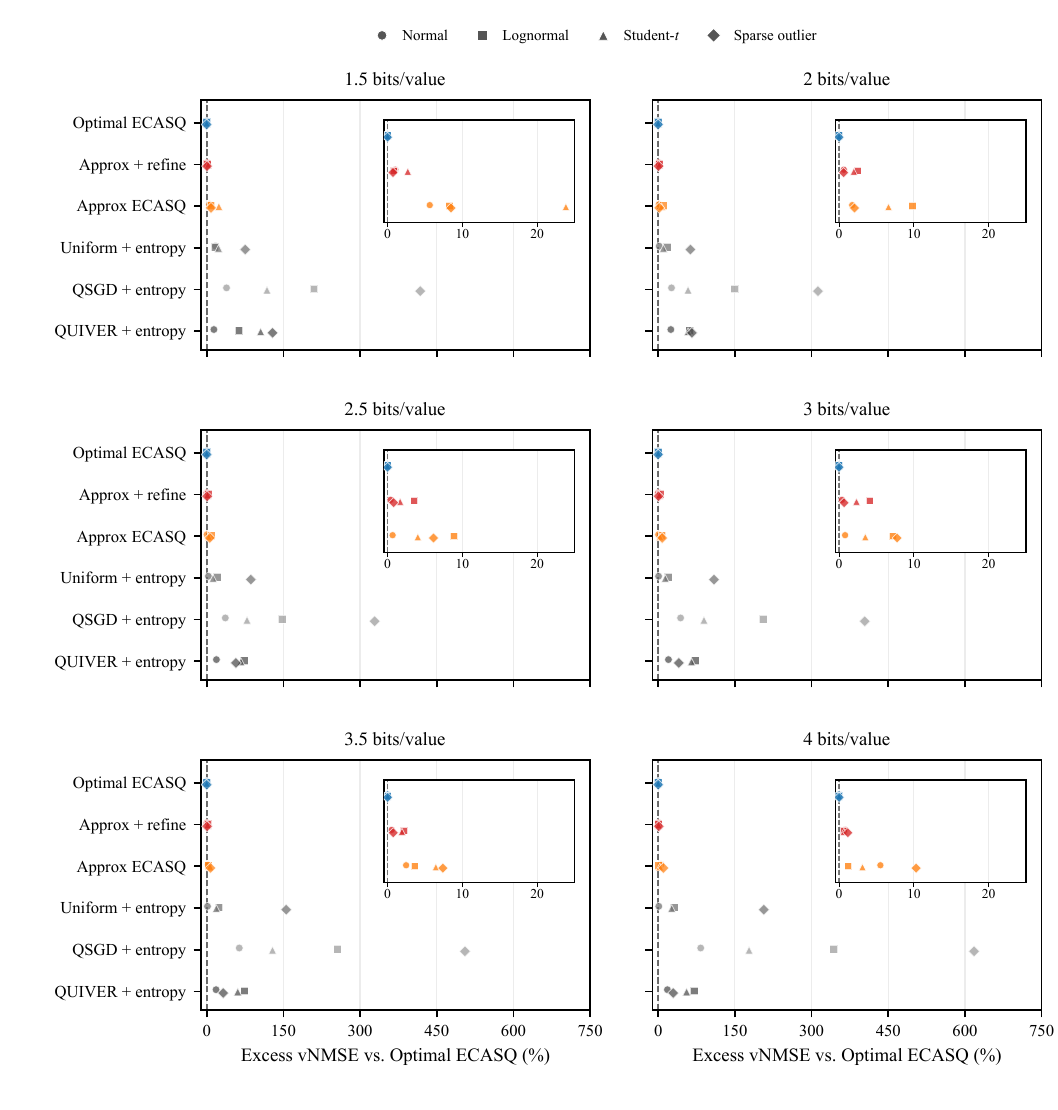}
    \caption{Synthetic matched-rate comparison at six total rates, using the
    same method colors and distribution markers as \Cref{fig:eval-vnmse}.
    Circles, squares, triangles, and diamonds denote Normal, Lognormal,
    Student-$t$, and sparse-outlier distributions, respectively. Insets magnify
    the first three rows on a shared $0$--$25\%$ scale.  Total rate includes output entropy and
    the 32-bit per-level metadata.  Smaller values nearer zero are better.}
    \label{fig:eval-synthetic-summary}
\evalfigureend

\subsection{Approximation and Refinement Reality vs. Guarantee}
\Cref{fig:eval-approx-refine} isolates the entropy-surrogate effect.  For an
Approx ECASQ codebook, replacing the optimized surrogate
$H(I,B)$ with the exact output entropy $H(\widehat X)$ preserves vNMSE and saves
exactly $H(B\mid \widehat X)$ bits/value; subfigure (a) shows how this saving
varies with the Approx surrogate budget $b$.  All 100 interpolated
distribution/seed/budget checks satisfy the empirical one-bit guarantee, and
$H(B\mid \widehat X)$ remains within the \mbox{proved one-bit bound.}

\Cref{fig:eval-refinement} reports refinement separately.  Starting from
Approx ECASQ, five rounds reduce the median objective gap to
$0.12\%$ (maximum $0.58\%$ over the 20 distribution/seed cases), with most
of the improvement occurring in the first few rounds.  Runtime grows only
slightly from 10 to 20 requested rounds because the round count is an upper
bound: refinement terminates as soon as a complete round fails to decrease the
objective, so most runs have converged before reaching 20.  The plotted timing
summary takes the running maximum for each distribution/seed trajectory before
computing the median, making the end-to-end curve monotone while suppressing
small measurement fluctuations \mbox{between independent runs.}

\evalfigurestart
    \centering
    \begin{minipage}[t]{0.48992\textwidth}
      \centering
      \includegraphics[width=\linewidth]{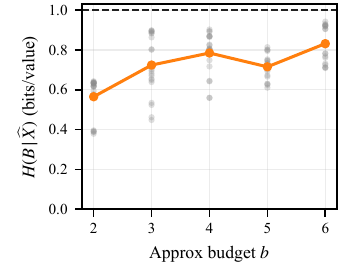}
      \par\smallskip{\small (a) Exact-entropy savings by Approx budget.}
    \end{minipage}
    \begin{minipage}[t]{0.48992\textwidth}
      \centering
      \includegraphics[width=\linewidth]{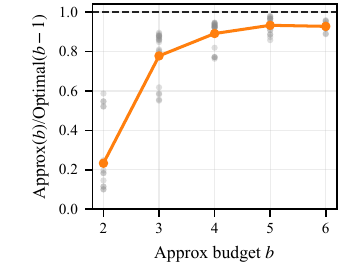}
      \par\smallskip{\small (b) Empirical one-bit guarantee.}
    \end{minipage}
    \caption{Approximation diagnostics over the 20 distribution--seed cases at
    each budget. Gray points show individual cases. Orange circles show the
    arithmetic mean in subfigure (a) and the geometric mean in subfigure (b). The
    dashed horizontal line marks the proved upper bound of one. Subfigure (a)
    reports the bits saved by exact output-entropy accounting for the same
    Approx ECASQ codebook selected under surrogate budget $b$, at unchanged
    vNMSE:
    $H(I,B)-H(\widehat X)=H(B\mid\widehat X)$. Subfigure (b) reports the vNMSE ratio
    between the surrogate-feasible Approx solution at budget $b$ and Optimal
    ECASQ at budget $b-1$. Both subfigures linearly interpolate the lower convex
    envelopes of the measured Lagrangian sweeps, corresponding to whole-vector
    time sharing between neighboring retained frontier points.}
    \label{fig:eval-approx-refine}
\evalfigureend

\evalfigurestart
    \centering
    \begin{minipage}[t]{0.49\textwidth}
      \centering
      \includegraphics[width=\linewidth]{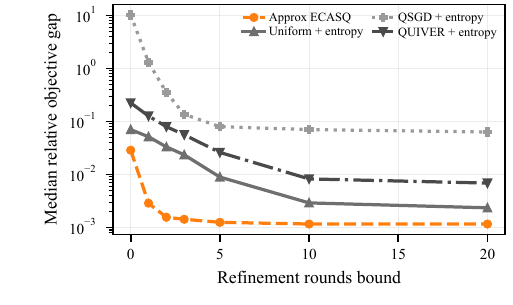}
      \par\smallskip{\small (a) Convergence to Optimal ECASQ.}
    \end{minipage}\hfill
    \begin{minipage}[t]{0.49\textwidth}
      \centering
      \includegraphics[width=\linewidth]{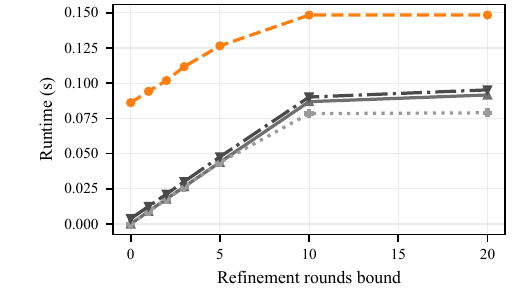}
      \par\smallskip{\small (b) End-to-end quantizer time.}
    \end{minipage}
    \caption{Refinement from four initializations.  Subfigure (a) shows the median
    objective gap to Optimal ECASQ.  Subfigure (b) shows median end-to-end runtime,
    including initialization and all executed refinement rounds; requested
    rounds are an upper bound because refinement stops after a round without
    objective improvement.  To reflect nondecreasing work despite run-to-run
    timing noise, each trajectory uses its running-maximum time before
    aggregation.}
    \label{fig:eval-refinement}
\evalfigureend

\subsection{Runtime and Objective Scaling}
We measure the cached Approx implementation on a 12-core Apple M2 Max with
32 GB of memory, using macOS 15.7.2 (arm64), Python 3.9.6, PyTorch 2.8.0, and
eight CPU threads.  
histogram construction and method-specific
setup, optimization, and refinement; it excludes random-vector generation,
metric evaluation, and the shared final codebook cast.  Each appendix point is
the mean over five independently generated vectors, with a two-sided 95\%
Student-$t$ confidence interval.  Optimal ECASQ and the two approximate variants
receive the same \mbox{weighted BF16 histogram.}

The main-paper runtime result in \Cref{fig:eval-runtime} aggregates the four
distribution-specific means; its error bars show their min--max range rather
than treating the distributions as exchangeable samples.  At $d=262{,}144$
and $s=64$, the pooled mean runtimes
are $6.39$ s for Optimal ECASQ, $0.26$ s for Approx, and $0.33$ s for refined Approx,
corresponding to $24.4\times$ and $19.2\times$ speedups.  Across the four
distributions, the respective speedup ranges are $13.4$--$34.3\times$ and
$8.7$--$28.3\times$.  These timing experiments hold $\lambda=0.1$ fixed.  The
lower subfigures therefore report the corresponding Lagrangian objective
$\vnmse(Q,X)+\lambda H(\widehat X)$ rather than a matched-rate distortion.  The
fast cached implementation is distinct from the streamed, linear-space
variant analyzed in the main text.  \Cref{fig:eval-runtime-s16} reports the
$s=16$ codebook-size setting, \Cref{fig:eval-runtime-s64} gives the full
per-distribution scaling curves for the main-paper setting, and
\Cref{fig:eval-runtime-s256} reports \mbox{the $s=256$ setting.}

\evalfigurestart
    \centering
    \includegraphics[width=0.98\textwidth]{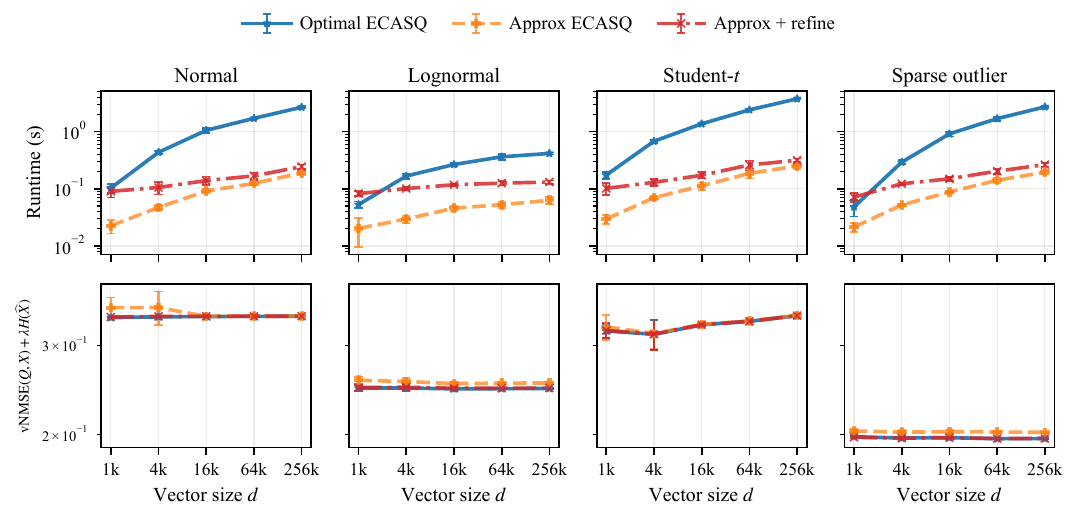}
    \caption{Runtime and objective scaling for BF16 inputs with at most $s=16$
    quantization values and fixed $\lambda=0.1$.  Points show the mean and two-sided 95\%
    Student-$t$ confidence interval over five seeds.  The lower subfigures report
    $\vnmse(Q,X)+\lambda H(\widehat X)$.  All runtime subfigures share one y-axis
    range, and all objective subfigures share another.}
    \label{fig:eval-runtime-s16}
\evalfigureend

\evalfigurestart
    \centering
    \includegraphics[width=0.98\textwidth]{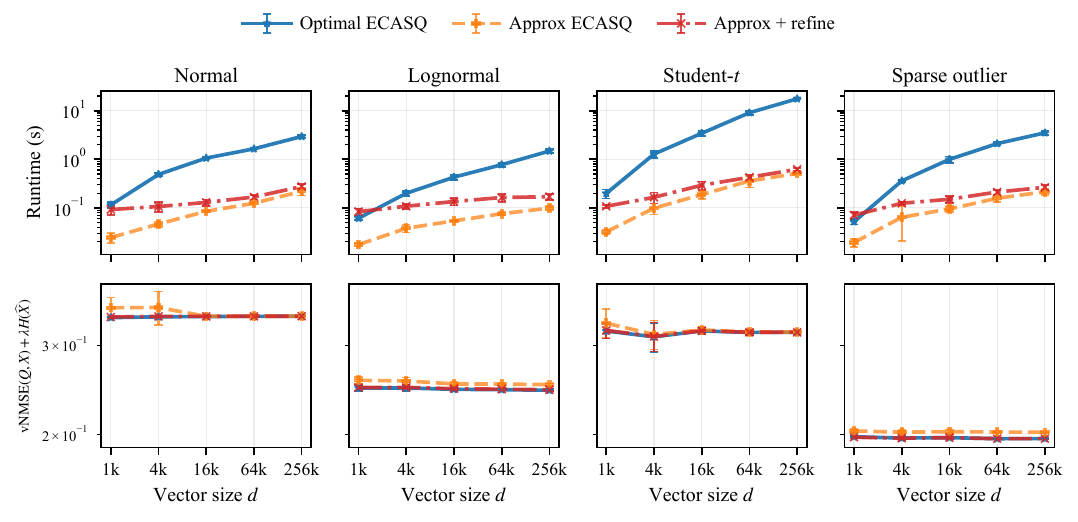}
    \caption{Runtime and objective scaling for BF16 inputs with at most $s=64$
    quantization values and fixed $\lambda=0.1$.  Points show the mean and two-sided 95\%
    Student-$t$ confidence interval over five seeds.  The lower subfigures report
    $\vnmse(Q,X)+\lambda H(\widehat X)$.  All runtime subfigures share one y-axis
    range, and all objective subfigures share another.}
    \label{fig:eval-runtime-s64}
\evalfigureend

\evalfigurestart
    \centering
    \includegraphics[width=0.98\textwidth]{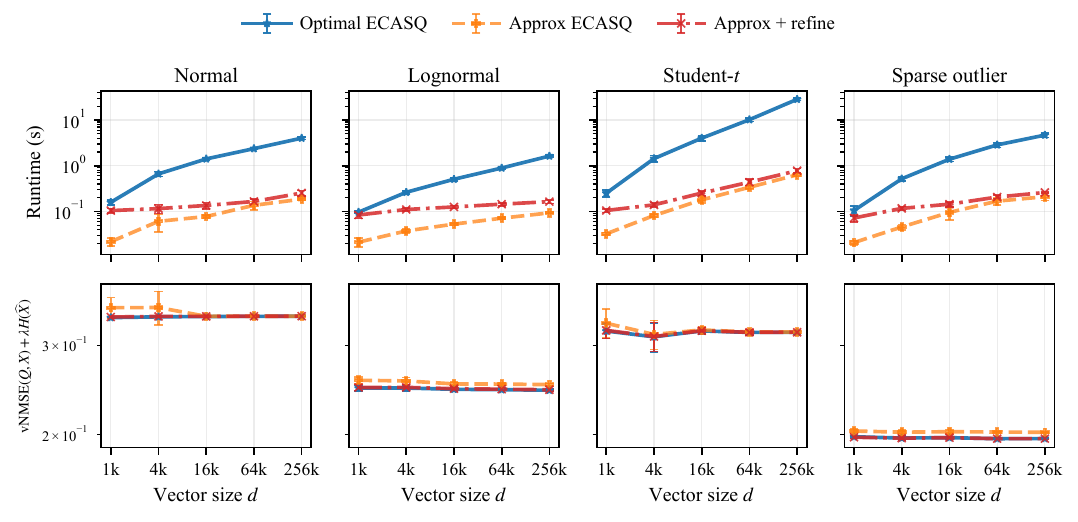}
    \caption{Runtime and objective scaling for BF16 inputs with at most $s=256$
    quantization values and fixed $\lambda=0.1$.  Points show the mean and two-sided 95\%
    Student-$t$ confidence interval over five seeds.  The lower subfigures report
    $\vnmse(Q,X)+\lambda H(\widehat X)$.  All runtime subfigures share one y-axis
    range, and all objective subfigures share another.}
    \label{fig:eval-runtime-s256}
\evalfigureend

\evalfigurestart
    \centering
    \includegraphics[width=0.98\textwidth]
      {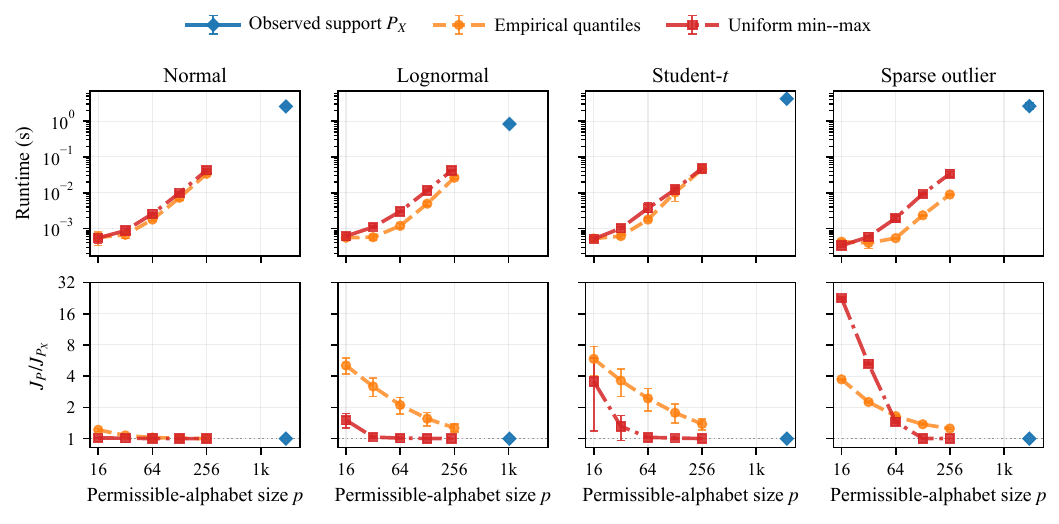}
    \caption{Runtime--quality tradeoff induced by the permissible alphabet
    $P$.  Each input is a BF16 vector $X\in\mathbb{R}^d$ with
    $d=16{,}384$; columns show Normal, Lognormal, Student-$t$, and
    sparse-outlier inputs.  We use at most $s=16$ quantization values,
    $\lambda=0.1$, requested candidate counts
    $p\in\{16,32,64,128,256\}$.  The upper row reports
    runtime, including histogram construction, candidate-set
    construction, and optimization, but excluding vector generation and
    metric serialization.  The lower row reports $J_P/J_{P_X}$, where
    $J_P=\vnmse(Q_P,X)+\lambda H(\widehat X)$.  Points and error bars show the
    mean and two-sided 95\% Student-$t$ confidence interval.  The horizontal
    axis uses the actual candidate count after duplicate removal.}
    \label{fig:eval-permissible-alphabet}
\evalfigureend

\evalfigurestart
    \centering
    \includegraphics[width=0.98\textwidth]
      {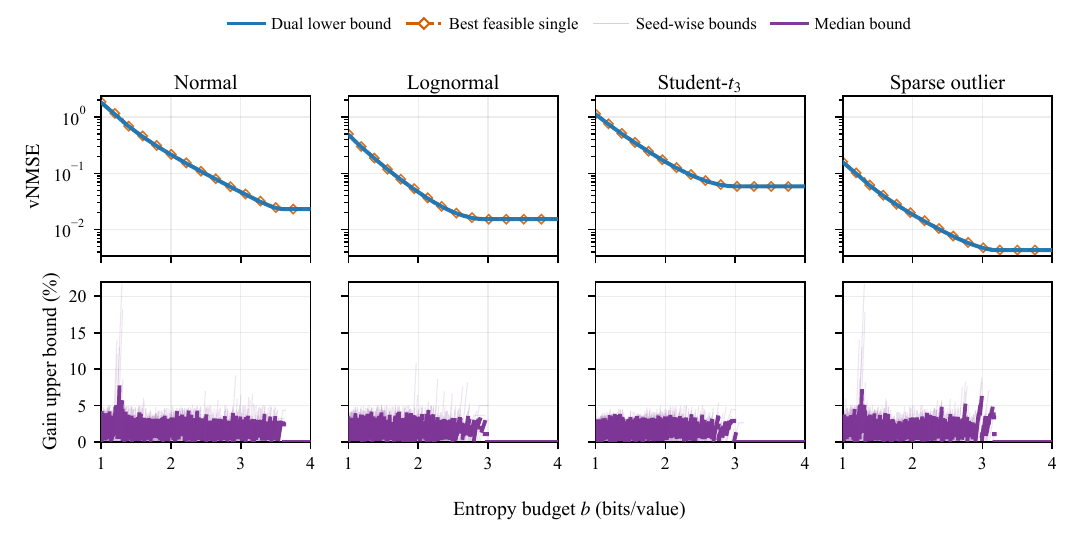}
    \caption{Entropy-constrained vNMSE and an upper bound on the
    benefit of two-way time sharing.  Columns show BF16 Normal, Lognormal,
    Student-$t$, and sparse-outlier vectors with $d=16{,}384$, $P=P_X$, at
    most $s=16$ quantization values, and entropy budgets
    $b\in[1,4]$ bits/value, sampled every $0.001$ bits and augmented with
    retained feasible-frontier breakpoints.  For each input, Optimal ECASQ
    solutions supported by $\vnmse(Q,X)+\lambda H(\widehat X)$ are adaptively
    refined from $\lambda=0$ through an endpoint attaining at most one
    bit/value.  Curves
    in the upper row are arithmetic means over the five seeds: solid lines show
    the dual lower bound $L(b)$, while dashed steps and open diamonds show the
    best retained feasible distortion $D_{\rm single,feasible}(b)$.  For each
    seed, these quantities give the relative-gain upper bound
    $U(b)=100[1-L(b)/D_{\rm single,feasible}(b)]$ shown in the lower row.  Thin
    lines are the five seeds and the thick line is their pointwise median.  All
    lower subfigures here and in
    \Cref{fig:eval-entropy-frontier,fig:eval-entropy-frontier-s256} use the
    same $0$--$22\%$ linear scale.
    Entropy is $H(\widehat X)$ and excludes reconstruction-value and
    entropy-model metadata, global framing, and finite-stream redundancy.}
    \label{fig:eval-entropy-frontier-s16}
\evalfigureend

\subsection{Effect of the Permissible Alphabet $P$}
\label{app:eval-permissible-alphabet}

The general-$P$ algorithm in \Cref{app:general-p} permits quantization values
that need not occur in the input.  We therefore compare the observed support
$
  P_X=\set{x_i:i\in[d]}
$
with smaller candidate alphabets.  For BF16 vectors, the alternatives are
weighted empirical quantiles and BF16-rounded uniform grids between the
observed minimum and maximum; both include the extrema.  We request
$p\in\set{16,32,64,128,256}$ candidates, although duplicate removal can make
the actual value of $\lvert P\rvert$ smaller.  We use the same four
distributions and five seeds as the synthetic study, with $d=16{,}384$, at
most $s=16$ quantization values, and fixed $\lambda=0.1$.  Every point is
computed with the optimal Hirschberg grid DP.  Runtime includes histogram and
candidate-set construction and optimization, but excludes vector generation
and metric serialization.

Uniform grids give the strongest overall tradeoff.  With 256 requested
candidates, their mean objective is within $0.25\%$ of the observed-support
result for every distribution while reducing runtime by $20$--$88\times$.
Empirical quantiles are competitive for Normal inputs, but at 256 candidates
their objective penalty is $26.9\%$, $38.5\%$, and $25.2\%$ on Lognormal,
Student-$t$, and sparse-outlier inputs, respectively.  Equal-mass quantiles
devote few candidates to tails and isolated extremes, whereas the uniform
min--max grid preserves range coverage.

\subsection{Entropy--vNMSE Frontiers and Time Sharing}
\label{app:eval-entropy-frontier}

We next measure how much two-way time sharing can improve the
entropy-constrained frontier for the four synthetic distributions.  Write
$D(Q)=\vnmse(Q,X)$ and $R(Q)=H(\widehat X_Q)$.  For a maximum codebook size
$s$ and retained set of multipliers $\Lambda$, define the best retained
feasible single quantizer and the dual lower bound
\[
\begin{aligned}
D_{\rm single,feasible}(b)
  &=\min_{\lambda\in\Lambda:\,R(Q_\lambda)\le b}D(Q_\lambda),\\
\phi(\lambda)
  &=\min_{Q\subseteq P_X,\,|Q|\le s}
      \{D(Q)+\lambda R(Q)\},\\
L(b)&=\max_{\lambda\in\Lambda}\{\phi(\lambda)-\lambda b\}.
\end{aligned}
\]
Every retained $\phi(\lambda)$ is computed by Optimal ECASQ, which globally
minimizes the Lagrangian objective.  For any deterministic quantizer or
time-share with rate at most $b$, its distortion is at least
$\phi(\lambda)-\lambda b$.  Consequently,
\[
L(b)\le D_{\rm TS}^\star(b)\le D_{\rm det}^\star(b)
     \le D_{\rm single,feasible}(b),
\]
where $D_{\rm TS}^\star$ and $D_{\rm det}^\star$ are the optimal time-shared
and deterministic frontiers.  The lower subfigures \mbox{therefore report the
upper bound}
\[
U(b)=100\left(1-\frac{L(b)}{D_{\rm single,feasible}(b)}\right)
\]
on the percentage improvement that time sharing can provide over the optimal
deterministic single quantizer.  We use output entropy here, rather than the
metadata-inclusive total rate, because this is the constraint convexified
by the time-sharing result in \Cref{sec:discussion}.

For numerical conservatism, before forming $L(b)$ we reduce every stored
Lagrangian objective $\phi$ by
$\max\{2\mathord\times10^{-7},2\mathord\times10^{-5}|\phi|\}$, the solver's
objective-validation tolerance.  This can only lower $L(b)$ and hence loosen
the reported upper bound.

We adaptively refine the supported frontier.  Starting from $\lambda=0$ and an
endpoint with entropy at most one bit/value, each adjacent discovered pair is
queried at the multiplier given by its supporting-line slope.  An interval is
closed when this query finds no new supported point, or when the relative
vNMSE change between its endpoints is at most $5\%$.  Above the entropy of the
$\lambda=0$ endpoint, the curves are extended constantly because that endpoint
already globally minimizes distortion.  On the $0.001$-bit budget grids, for
$s=16,64,256$, respectively, $U(b)$ has medians
$1.16\%,1.59\%,1.96\%$, $95$th percentiles $4.07\%,5.03\%,6.12\%$, and is at
most $5\%$ in $98.2\%,94.9\%,91.4\%$ of the
60k, 100k, and 140k distribution/seed/budget checks.  The maximum is
$21.7\%$ for each quantization-value budget and occurs on a narrow low-rate interval.
Thus, time sharing completes the theory, while the exact Lagrangian objectives
show that its possible improvement is small at most tested budgets.  This
upper bound relies on global optimality of the Lagrangian solves: the objective
of an arbitrary feasible or approximate solution is an upper bound on
$\phi(\lambda)$ and cannot be substituted into $L(b)$.

\evalfigurestart
    \centering
    \includegraphics[width=0.98\textwidth]
      {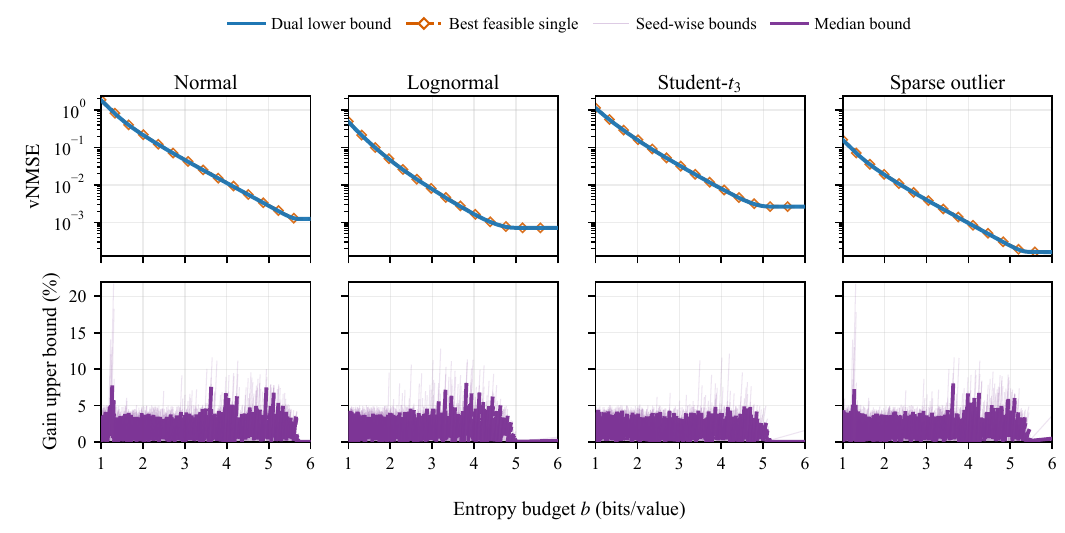}
    \caption{Entropy-constrained vNMSE and the time-sharing gain upper bound
    with at most $s=64$ quantization values.  The setup and curve encodings
    match \Cref{fig:eval-entropy-frontier-s16}: columns show the four BF16 synthetic
    distributions with $d=16{,}384$, $P=P_X$; budgets
    $b\in[1,6]$ bits/value are sampled every $0.001$ bits and augmented with
    retained feasible-frontier breakpoints.  The upper row shows the mean dual
    lower bound and best retained feasible distortion, and the lower row shows
    the five seed-wise gain upper bounds and their pointwise median.}
    \label{fig:eval-entropy-frontier}
\evalfigureend

\evalfigurestart
    \centering
    \includegraphics[width=0.98\textwidth]
      {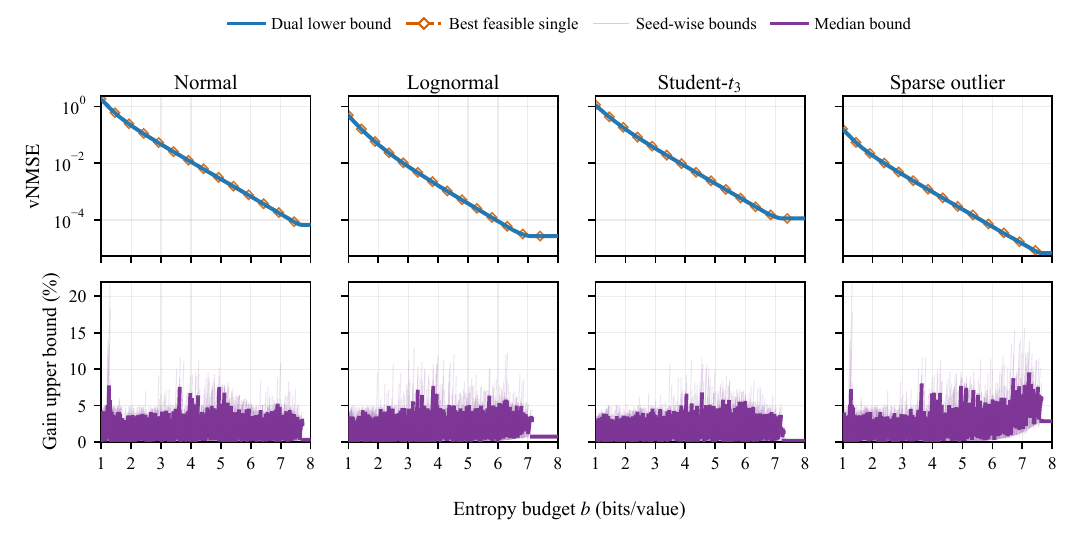}
    \caption{Entropy-constrained vNMSE and the time-sharing gain upper bound
    with at most $s=256$ quantization values.  The setup and curve encodings
    match \Cref{fig:eval-entropy-frontier-s16}: columns show the four BF16 synthetic
    distributions with $d=16{,}384$, $P=P_X$; budgets
    $b\in[1,8]$ bits/value are sampled every $0.001$ bits and augmented with
    retained feasible-frontier breakpoints.  The upper row shows the mean dual
    lower bound and best retained feasible distortion, and the lower row shows
    the five seed-wise gain upper bounds and their pointwise median.}
    \label{fig:eval-entropy-frontier-s256}
\evalfigureend
\FloatBarrier

\end{document}